\documentclass{article}

\usepackage{microtype}
\usepackage{graphicx}
\usepackage{subcaption}
\usepackage{booktabs}
\usepackage{amsmath}
\usepackage{amssymb}
\usepackage{amsthm}
\usepackage{mathtools}
\usepackage{natbib}
\usepackage{float}
\usepackage{enumitem}
\usepackage{bm}
\usepackage{xcolor}
\usepackage[UKenglish]{babel}
\usepackage{tikz}
\usepackage{xskak}
\usetikzlibrary{arrows.meta}
\usetikzlibrary{calc}
\DeclareMathAlphabet{\mathbbold}{U}{bbold}{m}{n}

\usepackage{hyperref}

\usepackage[accepted]{icml2026}

\newcommand{\E}{\mathbb{E}}
\newcommand{\T}{\mathsf{T}}
\renewcommand{\P}{\mathbb{P}}
\DeclareMathOperator*{\argmin}{argmin}
\newcommand{\R}{\mathbb{R}}
\newcommand{\N}{\mathbb{N}}
\newcommand{\de}{{d_\mathrm{e}}}
\newcommand{\dffn}{{d_\mathrm{ffn}}}
\newcommand{\din}{{d_\mathrm{in}}}
\newcommand{\dout}{{d_\mathrm{out}}}
\newcommand{\Block}{\mathrm{Block}}
\newcommand{\Attn}{\mathrm{Attn}}
\newcommand{\Read}{\mathrm{Read}}
\newcommand{\Embed}{\mathrm{Embed}}
\newcommand{\FFN}{\mathrm{FFN}}
\newcommand{\ReLU}{\mathrm{ReLU}}
\newcommand{\LocPol}{\mathrm{LocPol}}
\newcommand{\TF}{\mathrm{TF}}
\newcommand{\cH}{\mathcal{H}}
\newcommand{\cX}{\mathcal{X}}
\newcommand{\cD}{\mathcal{D}}
\newcommand{\cF}{\mathcal{F}}
\newcommand{\cT}{\mathcal{T}}
\newcommand{\diff}{\mathrm{d}}
\newcommand{\bK}{\bm{K}}
\newcommand{\bP}{\bm{P}}
\newcommand{\bI}{\bm{I}}
\newcommand{\bX}{\bm{X}}
\newcommand{\bH}{\bm{H}}
\newcommand{\bZ}{\bm{Z}}
\newcommand{\bQ}{\bm{Q}}
\newcommand{\bV}{\bm{V}}
\newcommand{\bW}{\bm{W}}
\newcommand{\btheta}{\bm{\theta}}
\newcommand{\op}{\mathrm{op}}

\newtoggle{journal}
\toggletrue{journal}

\theoremstyle{plain}
\newtheorem{theorem}{Theorem}[section]
\newtheorem{proposition}[theorem]{Proposition}
\newtheorem{lemma}[theorem]{Lemma}
\theoremstyle{definition}
\newtheorem{definition}[theorem]{Definition}

\icmltitlerunning{%
  Efficient and Minimax Optimal
  In-context
  Nonparametric Regression with Transformers
}

\begin{document}

\twocolumn[
  \icmltitle{%
    Efficient and
    Minimax Optimal In-context
    \texorpdfstring{\\}{}
    Nonparametric Regression with Transformers
  }



  \icmlsetsymbol{equal}{*}

  \begin{icmlauthorlist}
    \icmlauthor{Michelle Ching}{cambridge,equal}
    \icmlauthor{Ioana Popescu}{eth,equal}
    \icmlauthor{Nico Smith}{cambridge,equal}
    \icmlauthor{Tianyi Ma}{cambridge}
    \\
    \icmlauthor{William G.\ Underwood}{cambridge}
    \icmlauthor{Richard J.\ Samworth}{cambridge}
  \end{icmlauthorlist}

  \icmlaffiliation{cambridge}{%
    Statistical Laboratory,
    University of Cambridge,
    Cambridge, UK%
  }
  \icmlaffiliation{eth}{%
    Department of Computer Science,
    ETH Z{\"u}rich,
    Z{\"u}rich,
    Switzerland%
  }


  \icmlcorrespondingauthor{Tianyi Ma}{%
  tm681@cam.ac.uk}

  \icmlkeywords{In-context learning,
    nonparametric regression, transformers,
  approximation theory}

  \vskip 0.3in
]



\printAffiliationsAndNotice{\icmlEqualContribution}

\begin{abstract}

We study in-context learning for nonparametric regression with
$\alpha$-H\"older smooth regression functions, for some $\alpha>0$. We
prove that, with $n$ in-context examples and $d$-dimensional
regression covariates,
a pretrained
transformer with $\Theta(\log n)$ parameters and
$\Omega\bigl(n^{2\alpha/(2\alpha+d)}\log^3 n\bigr)$ pretraining sequences can
achieve the minimax optimal rate of convergence
$O\bigl(n^{-2\alpha/(2\alpha+d)}\bigr)$ in mean squared error.
Our result requires substantially fewer
transformer parameters and pretraining sequences
than previous results in the literature.
This is achieved by showing that
transformers are able to
approximate local polynomial estimators efficiently
by implementing a
kernel-weighted polynomial basis
and then running gradient descent.

\end{abstract}

\section{Introduction}

Deep learning models based on the transformer architecture
\citep{vaswani2017attention}
have achieved remarkable empirical successes in recent years;
prominent examples include large language models
\citep{devlin2018bert,shoeybi2019megatron,achiam2023gpt}
and contemporary computer vision models
\citep{dosovitskiy2021an}.
By allowing data points to interact directly via the
attention mechanism, such models enjoy a high degree
of flexibility,
while remaining trainable and avoiding overfitting.

In-context learning (ICL)
offers a framework for examining
the generalisation abilities of large language models
\citep{brown2020language,garg2022can}.
When presented with a prompt (context) containing a few
input--output examples,
pretrained transformers can often generalise to
new unseen queries without requiring any parameter updates.
Empirical studies have observed this behaviour
across
a diverse range of tasks, including translation,
answering questions, and arithmetic~\citep{brown2020language}.

\begin{figure}[ht]
  \centering
  \input{chess.tex}
  \caption{What is the relationship between these
  two concepts?}
  \label{fig:chess}
\end{figure}

As an example of ICL behaviour, consider
the concept-to-concept relationship depicted in
Figure~\ref{fig:chess}.
Here, a chess opening known as the
London System is paired with the iconic
London landmark Big Ben.
A language model might learn this type of relationship
via a text prompt as shown in Figure~\ref{fig:icl}.
This prompt includes several well-known chess openings
(London System, French Defence,
Spanish Game, King's Indian Defence, Scotch Game), paired with famous landmarks
from the corresponding countries.
As such, a suitable response for the query point
(the Italian Game) might be `Colosseum';
indeed, this is the response provided by OpenAI's GPT-5 mini model.
This large language model is able to identify the
abstract relationship
from chess opening to location to landmark, despite not
having been explicitly trained for this task.

\begin{figure}[ht]
  \centering
  \newlength{\iclstart}
\iftoggle{journal}{\setlength{\iclstart}{0cm}}{\setlength{\iclstart}{-0.7cm}}

\begin{tikzpicture}

  \node[anchor=mid west] at (\iclstart, 0.0)
  {1.\ d4 d5 2.\ Nf3 Nf6 3.\ Bf4};
  \draw[thick,-Latex] (4.5, 0.05) -- (5.2, 0.05);
  \node[anchor=mid west] at (5.6, 0.0)
  {Big Ben};

  \node[anchor=mid west] at (\iclstart, -0.6)
  {1.\ e4 e6 2.\ d4 d5};
  \draw[thick,-Latex] (4.5, -0.55) -- (5.2, -0.55);
  \node[anchor=mid west] at (5.6, -0.6)
  {Eiffel Tower};

  \node[anchor=mid west] at (\iclstart, -1.2) {1.\ e4 e5 2.\ Nf3 Nc6 3.\ Bb5};
  \draw[thick,-Latex] (4.5, -1.15) -- (5.2, -1.15);
  \node[anchor=mid west] at (5.6, -1.2) {Sagrada Fam\'ilia};

  \node[anchor=mid west] at (\iclstart, -1.8) {1.\ d4 Nf6 2.\ c4 g6};
  \draw[thick,-Latex] (4.5, -1.75) -- (5.2, -1.75);
  \node[anchor=mid west] at (5.6, -1.8) {Taj Mahal};

  \node[anchor=mid west] at (\iclstart, -2.4) {1.\ e4 e5 2.\ Nf3 Nc6 3.\ d4};
  \draw[thick,-Latex] (4.5, -2.35) -- (5.2, -2.35);
  \node[anchor=mid west] at (5.6, -2.4) {Loch Ness};

  \node[anchor=mid west] at (\iclstart, -3.0) {1.\ e4 e5 2.\ Nf3 Nc6 3.\ Bc4};
  \draw[thick,-Latex] (4.5, -2.95) -- (5.2, -2.95);
  \node[anchor=mid west] at (5.6, -3.0) {?};

\end{tikzpicture}
  \caption{Example of ICL in large
    language models.
    Here, there are $n=5$ in-context examples
    and one query point
  in the prompt.}
  \label{fig:icl}
\end{figure}

Recently, pretrained transformers have demonstrated the
ability to achieve state-of-the art
empirical performance on various tasks with tabular data via ICL
\citep{hollmann2025accurate},
outperforming classical approaches including
popular tree-based boosting algorithms.  One of the central such
challenges, and the focus of this work, is that of nonparametric regression.

\subsection{Related Work}

The theoretical properties of ICL have been extensively studied
in the regression and classification settings.
Several prior works have shown that transformers are able to
approximate classical algorithms and generalise to unseen tasks.
For example, transformers are able to learn
parametric models via gradient descent
\citep{
  ahn2023transformers,
  akyurek2023what,
  bai2023transformers,
  li2023transformers,
von2023transformers},
implement classifiers by acting as meta-optimisers \citep{dai2023gpt},
perform nonparametric regression using basis expansions
\citep{kim2024transformers} or
by fitting Nadaraya--Watson estimators
\citep{shen2025understanding}, and
solve reinforcement learning problems
\citep{lin2024transformers}.
Other papers study implicit Bayesian inference for next token
prediction \citep{xie2022an}, Bayesian frameworks for tabular
foundation models \citep{nagler2023statistical}, adaptivity and
distributional robustness of ICL
\citep{ma2025provable,wakayama2025in}, task generalisation
\citep{abedsoltan2025task}, prompt engineering \citep{nakada2025theoretical},
and transformer training dynamics \citep{oko2024pretrained,zhang2024trained,
kuwataka2025test,lu2025asymptotic}.
There have also been many empirical
studies of ICL phenomena in recent years.
For example,
\citet{bhattamishra2024understanding}
consider ICL of discrete functions,
\citet{qin2024factors}
explore the factors influencing
multi-modal
ICL in vision large language models,
\citet{sia2024where}
investigate where in the transformer
architecture ICL occurs,
and
\citet{bertsch2025context}
conduct a systematic study of
long-context ICL.

Closest to our work are the recent papers
by \citet{kim2024transformers} and \citet{shen2025understanding}.
\citet{kim2024transformers} prove that transformers
can learn nonparametric regression problems in context by estimating
the coefficients of the regression functions from a B-spline wavelet
basis expansion.
They achieve nearly minimax optimal rates
in~$n$, the number of in-context examples,
when the
regression functions belong to a suitable Besov space.
\citet{shen2025understanding} show that transformers are
able to approximate Nadaraya--Watson (local constant) estimators.
For $\alpha$-H\"older regression functions,
where $\alpha \in (0, 1]$,
they achieve optimal rates, and further they show
adaptation to covariates
supported on a low-dimensional manifold,
thereby avoiding the curse of dimensionality.
Both of these works require transformers with
polynomially many parameters (in $n$) in order
to achieve minimax optimal rates,
leading to restrictive assumptions
on the number of pretraining sequences.



\subsection{Contribution}

We consider the problem of in-context nonparametric regression
using transformers.
In particular, we show that transformers
trained via empirical risk minimisation,
under mild architecture conditions,
achieve the minimax optimal rate
(in the number $n$ of in-context examples)
for nonparametric regression
under mean squared prediction risk.
Our main result, Theorem~\ref{thm:main_result},
holds whenever the underlying regression functions
are $\alpha$-H{\"o}lder, for some $\alpha > 0$;
the corresponding minimax rate is
$n^{-2\alpha / (2\alpha + d)}$,
where $d$ is the dimension of the covariates.
Compared with existing works,
our approach attains this optimal rate while
imposing much milder conditions on both the
number of pretraining sequences and the number
of transformer parameters.
Our proofs use linear attention
for simplicity,
but they can be generalised to ReLU or softmax attention with minor changes;
see the discussion at the end
of Section~\ref{sec:main_results}.

Our primary technical innovation is to
show that transformers with single-head attention layers and $\Theta(\log n)$
parameters are able to
attain the optimal rate of convergence. This improves on  the
$\Theta(n^{d/(2\alpha+d)})$ parameters required by
\citet{kim2024transformers} and $\Theta(n)$ parameters needed by
\citet{shen2025understanding}.
As a direct consequence, we also
require substantially
fewer pretraining sequences.
In order to prove this,
we demonstrate that such transformers
are able to approximate local
polynomial estimators
of suitable degree
(Theorem~\ref{thm:approximation}), which amounts to solving a
weighted least squares problem
\eqref{eq:weighted-least-square}
in the monomial basis.
Our approach consists of
constructing a transformer which,
given an input prompt,
first constructs a kernel-weighted
monomial basis matrix
and then uses this matrix to perform
$\Theta(\log n)$
steps of gradient descent towards the desired
local polynomial
least-squares solution~\eqref{eq:weighted-least-square}.  An
advantage of gradient descent for this problem is that it avoids the
need to compute matrix inverses.  Optimality then follows by
combining well-known
properties of local polynomial
estimators (Theorem~\ref{thm:locpol_main})
with a bound on the expected excess risk
of the transformer-based estimator,
obtained via empirical process theory
and covering number bounds.

\subsection{Overview}

In Section~\ref{sec:setup}, we
present a mathematical formulation of our
in-context nonparametric regression problem
with H{\"o}lder-smooth regression functions,
define our transformer class using
linear attention and ReLU feed-forward layers,
and provide some auxiliary definitions and results
concerning classical local polynomial estimators
(Theorem~\ref{thm:locpol_main}).
Section~\ref{sec:main_results} gives our main
results, beginning with some novel
approximation theory for transformers in
Theorem~\ref{thm:approximation}.
Combining this with a bound on the expected excess
risk for empirical risk-minimising transformers
yields our primary contribution as
Theorem~\ref{thm:main_result}.
Here, we establish the minimax optimality
of transformers for performing in-context
nonparametric regression under standard H{\"o}lder
smoothness assumptions, and under
substantially weaker conditions on the number of transformer
parameters and on the number of pretraining sequences
than existing results in the literature.
In Section~\ref{sec:proof_strategy}, we provide some
insight into the main ideas underlying the proofs
of our main theorems.
First, we explain how to construct
a transformer that uses gradient descent to
approximate the output of a local polynomial estimator,
and second we bound the expected excess risk of our
estimator by controlling the covering numbers of
the corresponding transformer class.


\subsection{Notation}

We write $\N \coloneqq \{1, 2, \ldots\}$ and
$\N_0 \coloneqq \{0, 1, 2, \ldots\}$.
For $n \in \N$, we set $[n] \coloneqq \{1, \ldots, n\}$.
For $d \in \N$ and a multi-index $\nu \in \N_0^d$,
define
$|\nu| \coloneqq \sum_{j=1}^{d} \nu_j$
and $\nu! \coloneqq \prod_{j=1}^d \nu_j!$ and,
for a sufficiently smooth function $g: \R^d \to \R$,
define the order-$\nu$ partial derivative
$\partial_\nu g(x) \coloneqq
\partial^{|\nu|}g(x) \big/ \prod_{j=1}^d \partial x_j^{\nu_j}$.
For $x \in \R^d$, let
$\|x\|_2^2 \coloneqq \sum_{j=1}^d x_j^2$ and
$\|x\|_\infty \coloneqq \max_{j \in [d]} |x_j|$.
If $x, y \in \R$, we write
$x \land y \coloneqq \min\{x, y\}$ and
$x \lor y \coloneqq \max\{x, y\}$, and set
$x_+ \coloneqq x \lor 0$.
For a matrix $A$, we write
$\|A\|_{\op}$ for the $\ell_2$--$\ell_2$ operator norm, $\|A\|_{\max}$
for its maximum absolute entry,
and $\lambda_{\min}(A)$ and $\lambda_{\max}(A)$
for its minimum and maximum eigenvalues, respectively.
For non-negative sequences $(a_n)$
and $(b_n)$, we write
$a_n = O(b_n)$
if there exists $C > 0$
and $N \in \N$ such that
$a_n \leq C b_n$ for all $n \geq N$.
Similarly, we write $a_n = \Omega(b_n)$
if there exists $c > 0$
and $N \in \N$ such that
$a_n \geq c b_n$ for all $n \geq N$; thus $a_n = O(b_n)$
if and only if $b_n = \Omega(a_n)$.
If $a_n = O(b_n)$ and $a_n = \Omega(b_n)$, then we write $a_n = \Theta(b_n)$.
For a non-empty normed space
$(X, \|\cdot\|)$, $A \subseteq X$
and $\delta > 0$,
we say that a non-empty finite set
$A' \subseteq X$
is a \emph{$\delta$-cover}
of $A$ if
$\sup_{x \in A}
\min_{x' \in A'}
\|x - x'\| \leq \delta$.
We write $N(A, \delta, \|\cdot\|)$ for the minimal cardinality
of such a cover, when one exists.

\begin{definition}[H\"older-ball]
  \label{def:holder}
  Let $d \in \N$ and $\alpha,M>0$, and take
  $\underline{\alpha} \vcentcolon=
  \lceil\alpha\rceil - 1$ to be the largest
  integer strictly less than $\alpha$.
  Write
  $\cH(d,\alpha,M)$ for the set of $\underline{\alpha}$-times
  differentiable functions $g: [0, 1]^{d} \to [-M,M]$
  that satisfy
  \begin{equation*}
    \max_{\nu\in\N_0^d: |\nu|=\underline{\alpha}}
    \,
    \bigl|
    \partial_{\nu}g(x)-\partial_{\nu}g(x')
    \bigr|
    \leq M\|x - x'\|_{2}^{\alpha-\underline{\alpha}}
  \end{equation*}
  for $x, x' \in[0, 1]^{d}$.
\end{definition}

\section{Problem Set-up}
\label{sec:setup}

In this section, we formalise our
in-context nonparametric regression problem,
define our transformer class
and provide some auxiliary results
on local polynomial estimators.
In the ICL framework, pretraining data is modelled as a collection of
$\Gamma$ pretraining sequences, each generated according to a similar
(but not identical) mechanism. We consider an in-context regression
model, supposing that each such sequence, indexed by $\gamma \in
[\Gamma]$, consists of covariates $X_i^{(\gamma)} \in \mathbb{R}^d$
and responses $Y_i^{(\gamma)} \in \mathbb{R}$ for $i \in [n]$. We
assume that the covariates $X_i^{(\gamma)}$ and errors
$\varepsilon_i^{(\gamma)} \in \mathbb{R}$ are all drawn from a common
distribution; the responses are then generated as $Y_i^{(\gamma)} =
m^{(\gamma)}(X_i^{(\gamma)}) + \varepsilon_i^{(\gamma)}$, where the
regression functions $m^{(\gamma)}$
are drawn randomly from a distribution,
and therefore may vary across pretraining indices.

\subsection{Data Generating Mechanism}
\label{sec:data_generating_mechanism}

Let $P_{X,\varepsilon}$
be a distribution on
$[0, 1]^d \times [-1, 1]$ such that
if $(X, \varepsilon) \sim P_{X, \varepsilon}$
then $X$ admits a Lebesgue density function~$f_X$
with $0 < c_X \leq f_X(x) \leq C_X < \infty$
for all $x \in [0, 1]^d$
and $\varepsilon$ satisfies
$\E(\varepsilon \mid X) = 0$
almost surely with $\sigma^2 \coloneqq \E(\varepsilon^2) < \infty$.
Given $\alpha, M > 0$, let $P_{\cH}$ be a distribution on
$\cH(d,\alpha, M)$ with respect to the
Borel $\sigma$-algebra associated with
$\|\cdot\|_\infty$.
Now fix $\Gamma, n \in \N$.
For $\gamma \in [\Gamma]$, let
$m^{(\gamma)} \sim P_{\cH}$ be independent.
Further, for $\gamma \in [\Gamma]$ and $i \in [n+1]$, and conditional
on $(m^{(1)}, \ldots, m^{(\Gamma)})$,
let $\bigl(X^{(\gamma)}_i, \varepsilon^{(\gamma)}_i\bigr)
\sim P_{X, \varepsilon}$
be independent, and set
$Y^{(\gamma)}_i \vcentcolon= m^{(\gamma)}\bigl(X^{(\gamma)}_i\bigr)
+ \varepsilon^{(\gamma)}_i$.
For $\gamma\in[\Gamma]$, define
\begin{equation*}
  \cD^{(\gamma)}_n \coloneqq \bigl(X^{(\gamma)}_i,
  Y^{(\gamma)}_i\bigr)_{i\in[n]}.
\end{equation*}
For each $\gamma\in[\Gamma]$,
we try to predict $Y^{(\gamma)}_{n+1}$
given the $n$ examples $\cD^{(\gamma)}_n$
and a query $X^{(\gamma)}_{n+1}$.

\begin{figure}[ht]
  \centering
  \begin{tikzpicture}

  \draw[thick,rounded corners] (1.0, -1.8) rectangle (5.9, 3.7);

  \draw[thick,rounded corners] (1.5, -1.0) rectangle (5.4, 2);

  \draw (-0.2, 0.5) node[thick,circle,draw,minimum size=0.92cm,inner sep=0mm]
  (PXeps) {$P_{X,\varepsilon}$};
  \draw (-0.2, 2.9) node[thick,circle,draw,minimum size=0.92cm,inner sep=0mm]
  (PH) {$P_\cH$};

  \draw (4.5, 2.9) node[thick,circle,draw,minimum size=0.92cm,inner sep=0mm]
  (m) {$m^{(\gamma)}$};
  \draw[thick,-Latex, shorten <=1.4mm,shorten >=1.3mm] (PH) -- (m);

  \draw (2.6, 1.2) node[thick,circle,draw,minimum size=0.92cm,inner sep=0mm]
  (X) {$X_i^{(\gamma)}$};
  \draw (2.6, -0.2) node[thick,circle,draw,minimum size=0.92cm,inner sep=0mm]
  (eps) {$\varepsilon_i^{(\gamma)}$};
  \draw[thick,-Latex, shorten <=1.4mm,shorten >=1.3mm] (PXeps) -- (X);
  \draw[thick,-Latex, shorten <=1.4mm,shorten >=1.3mm] (PXeps) -- (eps);

  \draw (4.5, 0.5) node[thick,circle,draw,minimum size=0.92cm,inner sep=0mm]
  (Y) {$Y_i^{(\gamma)}$};
  \draw[thick,-Latex, shorten <=1.4mm,shorten >=1.3mm] (X) -- (Y);
  \draw[thick,-Latex, shorten <=1.4mm,shorten >=1.3mm] (eps) -- (Y);
  \draw[thick,-Latex, shorten <=1.0mm,shorten >=1.3mm] (m) -- (Y);

  \node at (4.45, -0.65) {$i \in [n+1]$};
  \node at (5.2, -1.47) {$\gamma \in [\Gamma]$};

\end{tikzpicture}
  \caption{Plate diagram showing the data generating mechanism
    of our in-context nonparametric regression problem.
    For each pretraining sequence index $\gamma \in [\Gamma]$,
    a regression function is drawn from $P_\cH$.
    Then, for each $i \in [n+1]$,
    i.i.d.\ covariates $X_i^{(\gamma)}$
    and errors $\varepsilon_i^{(\gamma)}$ are sampled from
    $P_{X, \varepsilon}$.
    Finally, responses are generated as
    $Y_i^{(\gamma)} = m^{(\gamma)}\bigl(X_i^{(\gamma)}\bigr)
    + \varepsilon_i^{(\gamma)}$.
  }
  \label{fig:plate}
\end{figure}

\subsection{Empirical and Population Risk}

Let $\cF$ be a class of functions from
$\bigl([0, 1]^d\times \R\bigr)^n \times [0,1]^d$ to $\R$
and define the \emph{empirical risk minimiser}
$\hat{f}_\Gamma \in \argmin_{f\in\cF} \hat{R}_{\Gamma}(f)$,
where
\begin{align}
  \label{eq:empirical_risk}
  \hat{R}_{\Gamma}(f)
  \coloneqq
  \frac{1}{\Gamma} \sum_{\gamma=1}^\Gamma
  \Bigl\{
    Y^{(\gamma)}_{n+1}
    - f\bigl(
      \cD^{(\gamma)}_n,X^{(\gamma)}_{n+1}
    \bigr)
  \Bigr\}^2.
\end{align}
The \emph{population risk} of a measurable function
$f: \bigl([0, 1]^d\times \R\bigr)^n \times [0,1]^d
\to \R$ is
\begin{align}
  \label{eq:population_risk}
  R(f) \coloneqq \E \Bigl(
    \bigl\{
      Y_{n+1} - f\bigl(\cD_n,X_{n+1}\bigr)
    \bigr\}^2
  \Bigr),
\end{align}
where $\cD_n\coloneqq (X_i,Y_i)_{i\in[n]}$ and
$(X_{n+1},Y_{n+1})$ are independent copies of $\cD^{(1)}_n$ and
$(X^{(1)}_{n+1},Y^{(1)}_{n+1})$
respectively.
Throughout the paper, we quantify performance of a procedure
by bounding its population risk.

\subsection{Transformer Class}

We take $\cF$ to be a parametrised class of transformer neural networks.
The architecture of such a network consists of a series of
\emph{transformer blocks}, each of which is composed of a (single-head)
\emph{linear attention layer} and a \emph{feed-forward network
layer}; see Figure~\ref{fig:block}.
We remark that our proofs would also carry over, with minor changes, for
ReLU attention, softmax attention or other attention mechanisms that
can approximate linear attention; see the discussion following
Theorem~\ref{thm:main_result}. However, in our main exposition, we restrict to
linear attention for simplicity.

\begin{definition}[Linear attention layer]
  \label{def:attention}
  Let $n, \de \in \N$ and
  $\bQ, \bK, \bV \in
  \R^{\de \times \de}$.
  Define $\Attn_{\bQ, \bK, \bV}: \R^{(n+1) \times \de}
  \to \R^{(n+1) \times \de}$
  by
  \begin{align*}
    \Attn_{\bQ, \bK, \bV}(\bZ)
    \coloneqq \bZ + \bZ \bQ
    (\bZ \bK)^\T \bZ \bV.
  \end{align*}
\end{definition}

The parameters $\bQ$, $\bK$ and $\bV$ are referred to as
the \emph{query}, \emph{key} and \emph{value} matrices, respectively.

%

\begin{definition}[Feed-forward network layer]
  \label{def:ffn}
  Let $n, \de, \dffn \in \N$ and
  $\bW_1\in\R^{\dffn \times \de}$,
  $\bW_2 \in \R^{\de \times \dffn}$,
  $b_1 \in \R^{\dffn}$
  and $b_2 \in \R^{\de}$.
  Define $\FFN_{\bW_1, \bW_2, b_1, b_2}
  : \R^{(n+1) \times \de} \to \R^{(n+1) \times \de}$ by
  \begin{align*}
    &\FFN_{\bW_1, \bW_2, b_1, b_2}
    (\bZ)
    \iftoggle{journal}{\\ &\quad}{}
    \coloneqq
    \bZ +
    \bigl\{
      \bW_2\ReLU(
      \bW_1 \bZ^\T + b_1 1_{n+1}^\T)
      + b_2 1_{n+1}^\T
    \bigr\}^\T,
  \end{align*}
  where $\ReLU: \R \to \R$ is defined by
  $\ReLU(x) \coloneqq x \lor 0$ and is applied entrywise.
\end{definition}

The architecture given in Definition~\ref{def:ffn}
is equivalent to applying a standard one-hidden layer
feed-forward neural network with ReLU activation and
skip connection \citep{he2016deep} to each row
in the input matrix $\bZ$.

\begin{definition}[Transformer]
  \label{def:transformer}
  Let $n, \de, \dffn, L \in \N$.
  For $\ell \in [L]$, take parameter vectors
  $\btheta^{(\ell)}
  \coloneqq
  \bigl(
    \bQ^{(\ell)},
    \bK^{(\ell)},
    \bV^{(\ell)},
    \bW_1^{(\ell)},
    \bW_2^{(\ell)},
    b_1^{(\ell)},
    b_2^{(\ell)}
  \bigr)
  \in
  \R^{\de \times \de}
  \times
  \R^{\de \times \de}
  \times
  \R^{\de \times \de}
  \times
  \R^{\dffn \times \de}
  \times
  \R^{\de \times \dffn }
  \times
  \R^{\dffn}
  \times
  \R^{\de}$
  and define
  $\Block_{\btheta^{(\ell)}}:
  \R^{(n+1) \times \de}
  \to \R^{(n+1) \times \de}$
  by
  \begin{align*}
    &\Block_{\btheta^{(\ell)}}
    (\bZ)
    \iftoggle{journal}{\\ &\quad}{}
    \coloneqq
    \FFN_{\bW_1^{(\ell)},
    \bW_2^{(\ell)}, b_1^{(\ell)}, b_2^{(\ell)}} \circ
    \Attn_{\bQ^{(\ell)}, \bK^{(\ell)}, \bV^{(\ell)}}
    (\bZ).
  \end{align*}
  Let $\btheta\coloneqq (\btheta^{(\ell)})_{\ell=1}^L$ and define
  the \emph{transformer}
  $\TF_{\btheta}: \R^{(n+1) \times \de}
  \to \R^{(n+1) \times \de}$
  by
  \begin{align}
    \label{eq:TF}
    \TF_{\btheta}
    (\bZ)
    &\coloneqq
    \Block_{\btheta^{(L)}}
    \circ \cdots \circ
    \Block_{\btheta^{(1)}}
    (\bZ).
  \end{align}
  Finally, define
  $\cT(\de, \dffn, L, B)$ to be the collection of all
  transformers of the form~\eqref{eq:TF} with
  every entry of each parameter in
  $\btheta$ bounded in absolute value by $B>0$.
\end{definition}

The function
$\TF_{\btheta}$ depends on
the number $n$ of in-context examples
only via the dimension of its argument
$\bZ \in \R^{(n+1) \times \de}$;
its parametrisation is fully determined by
$\btheta$ and is independent of $n$.

\begin{figure}[ht]
  \centering
  \iftoggle{journal}{
  \newcommand{\blockscale}{1}
}{
  \newcommand{\blockscale}{0.8}
}

\begin{tikzpicture}[every node/.style={scale=\blockscale}]
  \draw[thick,rounded corners] (0.5, 0.15) rectangle (5.45, 3.15);

  \draw[thick,rounded corners] (1, 1.66) rectangle (2.6, 2.34);
  \node at (1.8, 2) {Attention};

  \draw[thick,rounded corners] (3.6, 2.45) rectangle (4.6, 2.95);
  \draw[thick,rounded corners] (3.6, 1.75) rectangle (4.6, 2.25);
  \draw[thick,rounded corners] (3.6, 0.35) rectangle (4.6, 0.85);
  \node at (4.1, 2.7) {FFN};
  \node at (4.1, 2.0) {FFN};
  \node at (4.14, 1.3) {\Large{$\cdots$}};
  \node at (4.1, 0.6) {FFN};

  \draw[thick,-Latex] (0.15, 2.0) -- (0.88, 2.0);

  \draw [thick,-Latex] (3.0, 2.4) .. controls (3.0, 2.7) ..  (3.45, 2.7);
  \draw [thick,-Latex] (2.75, 2.0) -- (3.45, 2.0);
  \draw [thick,-Latex] (3.0, 1.6) .. controls (3.0, 1.3) ..  (3.45, 1.3);
  \draw [thick,-Latex] (3.0, 0.9) .. controls (3.0, 0.6) ..  (3.45, 0.6);
  \draw [thick] (3.0, 2.4) -- (3.0, 0.9);

  \draw [thick] (4.75, 2.7) .. controls (5.1, 2.7) ..  (5.1, 2.4);
  \draw [thick,-Latex] (4.75, 2.0) -- (5.88, 2.0);
  \draw [thick] (4.75, 1.3) .. controls (5.1, 1.3) ..  (5.1, 1.6);
  \draw [thick] (4.75, 0.6) .. controls (5.1, 0.6) ..  (5.1, 0.9);
  \draw [thick] (5.1, 2.4) -- (5.1, 0.9);

  \node at (-0.35, 1.988) {\Large{$\cdots$}};
  \node at (6.4, 1.988) {\Large{$\cdots$}};
\end{tikzpicture}
  \caption{A single block in the transformer architecture,
    consisting of a single head of linear attention followed by
    a one-layer feed-forward neural network applied identically
  to every row.}
  \label{fig:block}
\end{figure}

\subsection{Transformer-based Estimator}
Let $\de\geq d+2$. For $(x_i, y_i)_{i \in [n]}\in[0,1]^d \times
\R$ and $x_{n+1}\in[0,1]^d$, define the embedding function
$\Embed_{\de}: \bigl([0,1]^d \times \R\bigr)^n \times [0,1]^d \to
\R^{(n+1)\times \de}$ by
\begin{align*}
  &\Embed_{\de}\bigl((x_i, y_i)_{i \in [n]}, x_{n+1}\bigr)
  \iftoggle{journal}{\\[1mm] &\qquad}{}
  \coloneqq
  \begin{pmatrix}
    x_1^{\T} & y_1 & 0_{\de-d-2}^\T & 0\\
    \vdots & \vdots & \vdots & \vdots\\
    x_n^{\T} & y_n & 0_{\de-d-2}^\T & 0\\[1mm]
    x_{n+1}^\T &0 & 0_{\de-d-2}^\T & 1
  \end{pmatrix} \in \R^{(n+1)\times \de}.
\end{align*}
We do not require any sophisticated
forms of positional encoding:
the $\Embed$ function merely identifies the test point
$i = n+1$ in its last column.

We also define $\Read_{M,d}: \R^{(n+1)\times \de} \to \R$ by
\begin{align*}
  \Read_{M,d}(\bZ) \coloneqq (-M)
  \vee \bZ_{n+1,d+1} \wedge M,
\end{align*}
where we recall that the
regression functions are bounded by $M$.

\begin{definition}[Transformer-based estimator]
  \label{def:estimator}
  Let $n, d, \de, \dffn, L, \Gamma \in \N$
  with $\de \geq d + 2$ and
  take $B > 0$.
  Consider the class of functions from
  $\bigl([0,1]^d \times \R\bigr)^n \times [0, 1]^d$
  to $\R$ given by
  $\cF(\de,\dffn,L,B,M) \coloneqq
  \bigl\{ \Read_{M,d} \circ f \circ \Embed_{\de}:
  f \in \cT(\de,\dffn,L,B)\bigr\}$
  and define the empirical
  risk minimiser
  \begin{equation*}
    \hat{f}_\Gamma
    \in\argmin_{f\in\cF(\de,\dffn,L,B,M)} \hat{R}_{\Gamma}(f),
  \end{equation*}
  where $\hat{R}_{\Gamma}$ is defined in~\eqref{eq:empirical_risk}.
  Thus, $\hat f_\Gamma$
  is trained on
  the data
  $\bigl(X_i^{(\gamma)}, Y_i^{(\gamma)}\bigr)_{i \in [n+1], \gamma
  \in [\Gamma]}$.
\end{definition}

We remark that, in practice,
computation of the global empirical risk minimiser
$\hat f_\Gamma$ is typically intractable, but 
gradient-based optimisation methods often yield good solutions.  
In this paper, we focus only on the
statistical aspects of in-context learning;
the analysis of transformer training dynamics
is an independent and challenging problem.

\subsection{Local Polynomial Estimators}
Let $K: \R^d \to [0, C_K]$ be a Borel measurable
function supported on $[-1, 1]^d$ satisfying
$|K(x) - K(x')| \leq L_K \|x-x'\|_\infty$
for all $x, x' \in \R^d$ and some $L_K > 0$.  We do not assume that
$K$ integrates to 1.  Assume that $K(x) \geq c_K$ for
$x \in [-c_K, c_K]^d$, where $c_K \in (0,1]$.
Let $h > 0$ be the bandwidth and define $K_h(\cdot) \coloneqq K(\cdot/h)/h^d$.
For $p \in \N_0$,
write $D \coloneqq \binom{d + p}{p}$
and let $P_h: \R^d \to \R^D$ be defined by
$P_h(x) \coloneqq \bigl(x^\nu / (\nu! h^{|\nu|}):
\nu \in \N_0^d, 0 \leq |\nu| \leq p\bigr) \in \R^D$,
with components ordered in increasing lexicographic ordering in $\nu$, so that
its first coordinate is $P_h(x)_1 = 1$.
For $x \in [0, 1]^d$, define the random diagonal matrix
$\bK_h(x) \in \R^{n \times n}$ with $i$th diagonal entry given by
$\bK_h(x)_{i i} \coloneqq K_h(X_i - x)$ for $i \in [n]$.
Let $\bP_h(x) \in \R^{n \times D}$ be defined by
$\bP_h(x)_{i j} \coloneqq P_h(X_i - x)_j$
for $i \in [n]$ and $j \in [D]$.
Define
$\tilde{\bX}
\coloneqq n^{-1/2}\bK_h(X_{n+1})^{1/2} \bP_h(X_{n+1})
\in \R^{n\times D}$ and
$\tilde{Y} \coloneqq n^{-1/2} \bK_h(X_{n+1})^{1/2} Y \in \R^n$,
where $Y \coloneqq (Y_1,\ldots,Y_n)^\top$.
Let
\begin{align}
  \nonumber
  w_*
  &\coloneqq
  \argmin_{w\in\R^D}
  \frac{1}{n}\sum_{i=1}^n K_h(X_i - X_{n+1})
  \iftoggle{journal}{\\[-3mm]\nonumber&\qquad\qquad\qquad\qquad\times}{}
  \bigl\{Y_i - w^\T P_h(X_i - X_{n+1})\bigr\}^2 \\[1mm]
  \label{eq:weighted-least-square}
  &\,=
  \argmin_{w\in\R^D}
  \bigl\| \tilde{Y} - \tilde{\bX} w \bigr\|_2^2.
\end{align}
Then the local polynomial estimator is the
first component $w_{*,1}$ of $w_*$ \citep{fan1996locpol}, and we
define the \emph{$M$-truncated local polynomial estimator} at $X_{n+1}$ as
\begin{align}
  f_\LocPol(\mathcal{D}_n,X_{n+1})
  \coloneqq (-M) \vee w_{*,1} \wedge M. \label{eq:f-locpol}
\end{align}

We first establish the minimax optimality
of the truncated local polynomial estimator for
H{\"o}lder-smooth functions.  For alternative versions of this result, see \citet{fan1996locpol}, \citet{gyorfi2002distribution},
\citet{tsybakov2009nonparametric}
and \citet{samworth2025modern}.

\begin{theorem}
  \label{thm:locpol_main}
  Let $f_\LocPol$ be the
  $M$-truncated local polynomial estimator
  defined in~\eqref{eq:f-locpol}
  with degree $p \coloneqq \lceil\alpha\rceil$
  and bandwidth $h \coloneqq n^{-1/(2\alpha+d)}$.
  There exists $C > 0$ depending only on
  $d$, $\alpha$, $M$,
  $c_X$, $C_X$,
  $c_K$, $C_K$ and $L_K$,
  such that
  %
  \begin{align*}
    R(f_\LocPol) - \sigma^2
    \leq C n^{-2\alpha/(2\alpha+d)}.
  \end{align*}
  Moreover, we have with probability at least
  $1-n^{C/(2\alpha+d)} \exp(-n^{2\alpha/(2\alpha+d)}/C)$ that
  \begin{align}
    C^{-1} \leq
    \lambda_{\min}\bigl(\tilde{\bX}^\T \tilde{\bX}\bigr)
    \leq
    \lambda_{\max}\bigl(\tilde{\bX}^\T \tilde{\bX}\bigr)
    \leq C. \label{eq:eigenvalue-bound}
  \end{align}
\end{theorem}

On the high-probability event
\eqref{eq:eigenvalue-bound},
$\tilde{\bX}^\T \tilde{\bX}$
is invertible and we have the closed form expression
$w_* = \bigl(\tilde{\bX}^\T \tilde{\bX}\bigr)^{-1}
\tilde{\bX}^\T \tilde{Y}$;
further, the objective function in the optimisation
problem~\eqref{eq:weighted-least-square}
is strongly convex.

\section{Main Results}\label{sec:main_results}

Our first main result asserts the existence of
a transformer that approximates a local
polynomial estimator.
The central idea
is to first approximate the kernel-weighted monomial basis matrix
$\tilde\bX$
and response vector $\tilde Y$,
and then run gradient descent (using attention layers) to find an approximate solution to the least squares
problem~\eqref{eq:weighted-least-square}. We provide
sketch proofs of our main results in Section~\ref{sec:proof_strategy};
complete proofs are deferred to the appendices.

\begin{theorem}
  \label{thm:approximation}
  Let $f_\LocPol$ be the
  $M$-truncated local polynomial estimator
  defined in~\eqref{eq:f-locpol}
  with degree $p \coloneqq \lceil\alpha\rceil$,
  kernel $K(x) \coloneqq (1-\|x\|_1)_+^2$ and bandwidth
  $h \coloneqq n^{-1/(2\alpha+d)}$.
  There exists $C>0$, depending only on
  $d$, $\alpha$, $M$, $c_X$ and $C_X$,
  such that if
  $D \coloneqq \binom{d+p}{p}$,
  $\de \coloneqq 2d + 2D + 5$,
  $\dffn \coloneqq 6(D+1)(14+p)$,
  $L\coloneqq \lceil C \log (en) \rceil$
  and $B \coloneqq C n^2$,
  then there exists a transformer
  $f_\TF \in \cF(\de,\dffn,L,B,M)$ satisfying
  \begin{equation*}
    \bigl|R(f_\TF) - R(f_\LocPol)\bigr|
    \leq \frac{C}{n}.
  \end{equation*}
\end{theorem}

The bound in Theorem~\ref{thm:approximation}
could be improved to $O(1/n^c)$ for any fixed
$c \geq 1$ by adjusting the constant $C$.
However, $O(1/n)$ suffices for our purposes.  The choice of kernel is
convenient (see the discussion in Section~\ref{sec:proof_strategy}),
but we would expect our results to carry over, with minor
modifications, to other commonly-used kernels.
The next theorem shows that a transformer trained by minimising the
empirical risk is minimax optimal.

\begin{theorem}
  \label{thm:main_result}
  Let $n,d\in\N$ and
  suppose that the data
  are generated according
  to Section~\ref{sec:data_generating_mechanism}.
  There exists $C > 0$ depending only on
  $d, \alpha, M, c_X$ and $C_X$, such that if
  $\hat{f}_\Gamma$ is constructed as in
  Definition~\ref{def:estimator} with
  $p \coloneqq \lceil\alpha\rceil$,
  $D \coloneqq \binom{d+p}{p}$,
  embedding dimension
  $\de \coloneqq 2d+2D+5$,
  FFN width
  $\dffn \coloneqq 6(D+1)(14+p)$,
  number of transformer blocks
  $L \coloneqq \lceil C \log (en) \rceil$,
  parameter bound
  $B \coloneqq C n^2$ and
  number of pretraining sequences
  $\Gamma \geq C n^{2\alpha/(2\alpha+d)}\log^3(e n)$, then
  \begin{align*}
    \E\bigl\{
      R\bigl(\hat{f}_\Gamma\bigr)
    \bigr\}
    - \sigma^2
    \leq C n^{-\frac{2\alpha}{2\alpha+d}}.
  \end{align*}
\end{theorem}


We achieve the minimax optimal rate by taking
$\Gamma = \Omega(n^{2\alpha/(2\alpha+d)} \log^3 n\bigr)$
and using $\Theta(\log n)$ transformer parameters.
In contrast, \citet{shen2025understanding} require
$\Gamma = \Omega \bigl(n^{(6\alpha+d)/(2\alpha+d)} \log n\bigr)$
and $\Theta(n)$ transformer parameters, while
\citet{kim2024transformers} need
$\Gamma = \Omega\bigl(n^{(2\alpha+2d)/(2\alpha+d)} \log n\bigr)$
and $\Theta\bigl(n^{d / (2\alpha+d)}\bigr)$
parameters.
These improvements are due to our new
approximation result, Theorem~\ref{thm:approximation},
which shows that transformers can efficiently
represent an approximate version of
local polynomial estimation via gradient descent.
In particular,
our use of transformer blocks permits
the approximation error to decrease
polynomially in $n$ while
the total number of parameters grows
only logarithmically.
This is because FFN layers can approximate polynomials exponentially
fast \citep{LuShenYangZhang2020}, and attention layers can implement
gradient descent
\citep{bai2023transformers,
von2023transformers}, which also converges exponentially quickly.
In the standard (i.e.\ not in-context) regression setting,
existing neural network regression constructions
need $\Omega\bigl(n^{d/(2\alpha+d)}\bigr)$ parameters
for minimax optimal estimation
\citep{schmidt2020nonparametric,LuShenYangZhang2020}.
In contrast,
the FFN layers in transformers
permit heavy parameter sharing
while the attention
layers allow data points to interact
with each other directly; see Figure~\ref{fig:block}.
We emphasise that, in contrast to some prior work, our main result
guarantees minimax optimality
regardless of the smoothness level $\alpha > 0$.
Computing the empirical risk minimiser $\hat f_\Gamma$
involves optimising over the parameters $\btheta$ of the
underlying transformer network.
In general, this is a non-convex optimisation problem, so
standard gradient-based approaches
such as Adam \citep{kingma2014adam}
and AdamW \citep{loshchilov2018decoupled}
are not generally guaranteed to converge to a global minimum.
However, our main result
(Theorem~\ref{thm:main_result})
holds if instead $\hat f_\Gamma$ is taken to be any
transformer achieving an empirical risk within
$O\bigl(n^{-2\alpha/(2\alpha+d)}\bigr)$
of the optimal value.
A detailed analysis of transformer training dynamics is
beyond the scope of this paper.


We remark that it is straightforward to adapt our results to certain other non-linear attention
mechanisms.
For instance,
$\ReLU(x) - \ReLU(-x) = x$, so two ReLU attention heads
can implement a single linear attention head.
It follows that Theorems~\ref{thm:approximation}
and~\ref{thm:main_result} hold for transformers with two ReLU
attention heads in each attention layer.
Moreover, it is possible to use two softmax attention layers followed
by $O(\log n)$ FFN layers to approximate the output of a single
linear attention layer; see Appendix~\ref{sec:softmax-approx-linear}.
Thus, similar results to our
Theorems~\ref{thm:approximation} and~\ref{thm:main_result}
hold for softmax transformers with
$\Theta(\log^2 n)$ layers.

\section{Proof Sketches}
\label{sec:proof_strategy}



The proof of our approximation result,
Theorem~\ref{thm:approximation},
proceeds via several steps; see Figure~\ref{fig:transformer_steps}.
First, we use three transformer blocks to construct
the centred and scaled covariates
$(\bX - 1_n X_{n+1}^\T) / h$,
along with the
square root of the diagonal kernel matrix $\bK_h(X_{n+1})$.
Here, it is convenient for us to use the specific kernel
$K(x) \coloneqq (1 - \|x\|_1)_+^2$ because its square root
is piecewise linear and so can be constructed exactly using
ReLU FFNs.
Next, we approximate the monomial basis
$\bP_h(X_{n+1})$, relying on the ability of
ReLU FFNs to approximate polynomials
with error $O(1/n^c)$ using $\Theta(\log n)$ layers
\citep{LuShenYangZhang2020}.
The monomial basis matrix and the responses are then
premultiplied by the square root kernel matrix,
yielding approximations of $\tilde \bX$ and $\tilde Y$.
Finally, the local polynomial estimator is given by
the solution to the least-squares optimisation problem
\eqref{eq:weighted-least-square}.
We obtain an approximately optimal solution
to this problem by
implementing gradient descent using transformers
\citep{bai2023transformers,
von2023transformers}.
We only need $\Theta(\log n)$ gradient descent steps
because the optimisation problem is strongly convex on the
high-probability
event~\eqref{eq:eigenvalue-bound}.
We keep track of the errors incurred,
first due to having access only to approximations
of $\tilde\bX$ and $\tilde Y$, and second because we
perform only finitely many steps of gradient descent. See
Appendix~\ref{sec:transformer_construction} for the proofs.

\begin{figure}[ht]
  \centering
  \scalebox{0.8}{
    \begingroup
\setlength\arraycolsep{3.5pt}

\newlength{\tfblockspacing}
\iftoggle{journal}{
  \setlength{\tfblockspacing}{0cm}
}{
  \setlength{\tfblockspacing}{0.5cm}
}

\begin{tikzpicture}

  \node at ($(0.1, 0) - (\tfblockspacing,0)$) {
    $
    \begin{pmatrix}
      \bX & Y \\
      X_{n+1}^\T & \cdot
    \end{pmatrix}
    $
  };

  \draw[thick,-Latex] (1.3, 0.0) -- (3.4, 0.0) node
  [anchor=south,midway,inner sep=1mm,xshift=-1mm] {Three blocks};

  \node at ($(6.0, 0) + (\tfblockspacing,0)$) {
    $
    \begin{pmatrix}
      \dfrac{\bX - 1_n X_{n+1}^\T}{h} &
      \sqrt{\dfrac{\bK_h(X_{n+1})}{n}}
    \end{pmatrix}
    $
  };

  \draw[thick,-Latex] (-0.7, -1.5) -- (1.8, -1.5) node
  [anchor=south,midway,inner sep=1mm,xshift=-1mm] {$\Theta(\log n)$ blocks};

  \node at ($(5.25, -1.5) + (\tfblockspacing,0)$) {
    $
    \begin{pmatrix}
      \!
      \sqrt{\dfrac{\bK_h(X_{n+1})}{n}} \bP_h(X_{n+1}) &
      \sqrt{\dfrac{\bK_h(X_{n+1})}{n}} Y
      \!
    \end{pmatrix}
    $
  };

  \draw[thick,-Latex] (-0.7, -2.8) -- (1.8, -2.8) node
  [anchor=south,midway,inner sep=1mm,xshift=-1mm] {$\Theta(\log n)$ blocks};

  \node at ($(2.27, -2.8) + (\tfblockspacing,0)$) {$w_*$};

\end{tikzpicture}

\endgroup
  }
  \caption{
    Key steps in the construction of the
    approximating transformer $f_\TF$;
    unchanged quantities are omitted for
    clarity.
    The first three blocks compute the centred and scaled
    covariates $\bX$, relative to the test point $X_{n+1}$,
    along with the kernel matrix $\bK_h(X_{n+1})$.
    Next, $\Theta(\log n)$ blocks are used to
    repeatedly multiply
    these, producing approximations
    of the kernel-weighted monomial basis
    $\tilde\bX$ and responses $\tilde Y$.
    Finally, $\Theta(\log n)$ steps of gradient
    descent are applied to arrive at an approximation
    of the optimal point $w_*$.
  }
  \label{fig:transformer_steps}
\end{figure}

Theorem~\ref{thm:main_result} is proved
using a decomposition of the population risk of
the empirical risk minimising transformer~$\hat f_\Gamma$; see
Figure~\ref{fig:proof_diagram}.
%
%
Since the transformer functions in
$\cF$ have at most $O(\log n)$ parameters, each bounded in magnitude
by $O(n^2)$, and as the transformer output
is Lipschitz in the parameters,
we are able to show (Lemma~\ref{lem:covering}) that
the covering numbers of this class satisfy
\begin{align*}
  \log N\bigl(\cF, \delta, \|\cdot\|_\infty \bigr)
  = O\biggl\{
    \log^3 n + (\log n)
    \log\biggl(\frac{1}{\delta}\biggr)
  \biggr\},
\end{align*}
where $\delta > 0$.
Applying a standard empirical process theory result
for $L_2$-empirical risk minimisers \citep{gyorfi2002distribution} yields the
following expected excess risk bound:
\begin{align*}
  &\E\bigl\{R(\hat f_\Gamma)\bigr\} - \sigma^2
  \leq
  2\bigl\{R(f_{\TF}) - R(f_{\LocPol})\bigr\}\\
  &\quad + 2\bigl\{R(f_{\LocPol}) - \sigma^2\bigr\}
  + O\biggl(
    \frac{\log^3 n + (\log n)\log\Gamma}{\Gamma}
  \biggr).
\end{align*}
The remaining terms are bounded by applying
Theorem~\ref{thm:approximation}
and Theorem~\ref{thm:locpol_main}.

\begin{figure}[ht]
  \centering
  \newlength{\proofspacing}
\iftoggle{journal}{
  \setlength{\proofspacing}{0cm}
}{
  \setlength{\proofspacing}{0.1cm}
}

\begin{tikzpicture}

  \draw [thick] (0.5, 2.0) .. controls (-2, 0.0)
  and (-3.5, 2.4) .. (-4.5, 4.3);
  \draw [thick] (-1.0, 4.3) .. controls (-3.2, 0.9)
  and (-5.0, 1.0) .. (-6.0, 1.5);
  \node at (-5.6, 1.7) {$R(f)$};
  \node at (-4.88, 3.85) {$\hat R_\Gamma(f)$};

  \draw[thick] ($(-0.3, -0.27)-(0,\proofspacing)$) rectangle (-3.1, 0.7);
  \draw[thick,-Latex] (-6, 0.4) -- (0.5, 0.4);
  \node at (-0.59, 0.02) {$\cF$};
  \node at (0.3, 0.02) {$f$};

  \fill (-1.3, 0.4) circle [radius=2.0pt] node [anchor=north] {$\hat f_\Gamma$};
  \fill (-2.4, 0.4) circle [radius=2.0pt] node [anchor=north]
  {$\vphantom{\hat f_\Gamma}f_\TF$};
  \fill (-4.96, 0.4) circle [radius=2.0pt] node [anchor=north] {$f_\LocPol$};

  \fill (-4.96, 1.24) circle [radius=2.0pt] node [anchor=north west] {(a)};
  \fill (-2.4, 2.52) circle [radius=2.0pt] node [anchor=north west] {(b)};
  \fill (-1.3, 1.25) circle [radius=2.0pt] node [anchor=south] {(c)};
  \fill (-1.3, 3.85) circle [radius=2.0pt] node [anchor=north west] {(d)};

  \draw[thick, dashed] (-1.3, 3.85) -- (-2.4, 3.85) -- (-2.4, 2.52);
  \draw[thick, dashed] (-2.4, 2.52) -- (-4.96, 2.52) -- (-4.96, 1.24);

\end{tikzpicture}
  \caption{
    Overview of the proof of Theorem~\ref{thm:main_result}.
    (a) The truncated local polynomial estimator
    attains the minimax optimal population risk of
    $R(f_\LocPol) - \sigma^2 =
    O\bigl(n^{-2 \alpha / (2 \alpha + d)}\bigr)$,
    by Theorem~\ref{thm:locpol_main}.
    (b) We construct a specific transformer with excess risk
    $R(f_\TF) - R(f_\LocPol) = O(1/n)$;
    see Theorem~\ref{thm:approximation}.
    (c) Our estimator $\hat f_\Gamma$ minimises the empirical
    risk $\hat R_\Gamma$ over the transformer class $\cF$
    (Definition~\ref{def:estimator}).
    (d) By bounding the covering numbers of $\cF$, we show that
    $\E\bigl\{R(\hat f_\Gamma)\bigr\} - 2 R(f_\TF) + \sigma^2 =
    O\bigl((\log^3 n + (\log n) \log \Gamma) / \Gamma\bigr)$;
    see Appendix~\ref{sec:covering}.
  }
  \label{fig:proof_diagram}
\end{figure}

\section{Simulations}

In this section, we compare the performance of a pretrained
transformer with local polynomial estimators. To this end, we
consider random regression tasks (drawn using a random Fourier series) with $d=3$, $\alpha=3$,
$\sigma=0.01$ and $n\in \{15,20,25,30,35\}$.  We use a
linear attention transformer with embedding dimension $\de = 256$,
FFN width $\dffn=1024$ and $L=12$ transformer blocks. The transformer
was trained for $50{,}000$ optimisation steps using AdamW \citep{loshchilov2019decoupled} with weight decay $10^{-3}$ and a cosine annealing learning rate schedule. At each optimisation step, we sampled 40 random regression functions, and
randomly generated $16$ pretraining sequences for each function (i.e.
  the covariate vectors and queries were random while the
regression function was held fixed), yielding a total of $640$ pretraining
sequences per optimisation step. The pretraining was carried out on an NVIDIA
A100 GPU (80GB) over approximately $2.25$ hours. At test time, for
each $n\in \{15,20,25,30,35\}$, we generated $10^5$ random regression
tasks, each with $n$ in-context examples, and compared the performance of our pretrained transformer with
a local cubic estimator with bandwidth
$h \in \{0.30, 0.40, 0.50, 0.60, 0.70, 0.80, 1.00, 1.25, 1.50, 2.00\}$
and ridge penalty\footnote{A
  local cubic estimator uses a monomial basis with dimension
  $\binom{6}{3} = 20$. Thus, in our simulations, we added a ridge penalty
  to~\eqref{eq:weighted-least-square} to improve its empirical
performance.} in $\{0.001, 0.005, 0.01, 0.05, 0.1, 0.5\}$ chosen via 5-fold
cross validation. Figure~\ref{fig:simulations} shows the estimated
excess risks for the pretrained transformer and the local polynomial estimator. We see that the transformer and the local polynomial estimator exhibit a similar dependence on~$n$, while the transformer outperforms the local polynomial estimator.
The code for our simulations is available at
\url{https://github.com/tianyima2000/ICL_LocPol}.

\begin{figure}
  \centering
  \includegraphics[width=\linewidth]{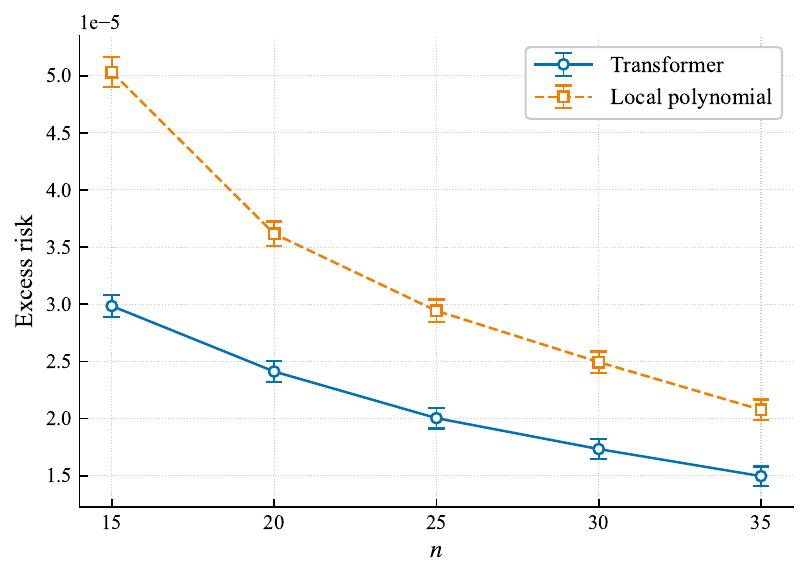}
  \caption{Estimated excess risks for a pretrained transformer and
  local polynomial estimator. Error bars represent $90\%$ confidence intervals for the excess population risks based on $10^5$ test sequences.}
  \label{fig:simulations}
\end{figure}

\section{Summary}

We have presented a theoretical study of
in-context learning for
nonparametric regression problems.
With $n$ in-context examples,
$d$-dimensional covariates and
$\alpha$-H{\"o}lder smooth
regression functions, we have shown that
suitable pretrained transformers with
only $\Theta(\log n)$ parameters are able
to attain the minimax optimal rate
$O\bigl(n^{-2\alpha / (2\alpha + d)}\bigr)$.
Moreover, we showed that this rate is
achievable whenever
$\Gamma = \Omega(n^{2\alpha/(2\alpha+d)} \log^3 n\bigr)$
pretraining sequences are available.
Our approach involved first demonstrating
that transformers are able to
approximate local polynomial estimators
to arbitrary accuracy,
and then bounding the risk of
our transformer estimator
by deriving covering number bounds
and applying results from
empirical process theory.

Future directions could include
extensions to next token prediction
in dependent data settings;
this would allow for more accurate modelling
of large language models.
Alternatively, one might attempt
to discover and exploit low-dimensional
structure (such as sparsity, an index structure
or a manifold hypothesis)
in the regression function,
thereby avoiding the curse of
dimensionality.
Finally, one could aim to establish
minimax theory for transformers trained
via gradient descent
(or related optimisation algorithms).

\section*{Acknowledgement}
The first three authors were funded by Summer Research in Mathematics bursaries from the University of Cambridge.  The last three authors were funded by RJS's European Research Council Advanced Grant 101019498.

\section*{Impact Statement}
This paper presents work whose goal is to advance the field of Machine
Learning. There are many potential societal consequences of our work, none
which we feel must be specifically highlighted here.

\bibliography{refs}
\bibliographystyle{icml2026}

\newpage
\appendix
\onecolumn

\section{Local Polynomial Estimation}
\label{sec:locpol}

We consider the classical degree-$p$
local polynomial estimator for
$d$-dimensional covariates
\citep{fan1996locpol,gyorfi2002distribution,hardle2004nonparametric,
tsybakov2009nonparametric},
truncating the
output to avoid overly poor estimation
on low-probability events.
Our main contribution in this section is to provide
high-probability bounds with sub-exponential tails
for the integrated squared error
of this estimator
(Theorem~\ref{thm:locpol}).
As an intermediate result, we also
provide a high-probability uniform guarantee on the
spectra of the associated
kernel-weighted monomial basis
matrices in Lemma~\ref{lem:locpol_Hhat},
which is also essential for showing that
transformers can successfully approximate
local polynomial methods.
We remark that truncation is important for
passing from high-probability guarantees to
bounds on the expected squared error:
see, for example,
\citet[Problem~20.4]{gyorfi2002distribution}.

\subsection{Local Polynomial Estimation Set-up}
\label{sec:locpol_setup}

Let $P_{X,\varepsilon}$ be as defined in
Section~\ref{sec:data_generating_mechanism}
and let $m \in \cH(d,\alpha, M)$ be fixed,
where $\alpha, M > 0$.
Let $(X_1, \varepsilon_1),\ldots,(X_n, \varepsilon_n)
\stackrel{\mathrm{i.i.d.}}{\sim} P_{X,\varepsilon}$ and define
\begin{equation*}
  Y_i \coloneqq m(X_i) + \varepsilon_i
\end{equation*}
for $i \in [n]$.
Suppose that $K: \R^d \to [0, C_K]$ is a Borel measurable
function supported on $[-1, 1]^d$ satisfying
$|K(x) - K(x')| \leq L_K \|x-x'\|_\infty$
for all $x, x' \in \R^d$.
Assume that $K(x) \geq c_K$ for each
$x \in [-c_K, c_K]^d$, where $c_K \in (0,1]$.
Let $h > 0$ and define $K_h(\cdot) \coloneqq K(\cdot/h)/h^d$.
For $p \in \N_0$ with $p \geq \underline{\alpha}$,
write $D \coloneqq \binom{d + p}{p}$
and let $P_h: \R^d \to \R^D$ be defined by
$P_h(x) \coloneqq \bigl(x^\nu / (\nu! h^{|\nu|}):
\nu \in \N_0^d, 0 \leq |\nu| \leq p\bigr)$,
with components ordered in increasing lexicographic ordering in $\nu$,
so that $P_h(x)^\T e_1 = 1$,
where $e_1$ denotes the first standard basis vector in $\R^D$.
We may index elements of $\R^D$ using either
regular indices $j \in [D]$
or multi-indices $\nu \in \N_0^d$, as convenient.
For $x \in [0, 1]^d$, define the random diagonal matrix
$\bK_h(x) \in \R^{n \times n}$ with $i$th diagonal entry given by
$\bK_h(x)_{i i} \coloneqq K_h(X_i - x)$ for $i \in [n]$.
Let $\bP_h(x) \in \R^{n \times D}$ be defined by
$\bP_h(x)_{i j} \coloneqq P_h(X_i - x)_j$
for $i \in [n]$ and $j \in [D]$.
Define $\hat\bH(x) \coloneqq \bP_h(x)^\T \bK_h(x) \bP_h(x) / n
\in \R^{D \times D}$
and $\bH(x) \coloneqq \E\bigl\{\hat\bH(x)\bigr\}
= \E\bigl\{P_h(X_1 - x) K_h(X_1 - x) P_h(X_1 - x)^\T\bigr\}$.
With $Y \coloneqq (Y_1, \ldots, Y_n)^\T$,
define the local polynomial estimator of degree~$p$ at $x$ by
$\hat m_n(x) \coloneqq e_1^\T \hat\bH(x)^{-1} \bP_h(x)^\T \bK_h(x) Y / n$
and let
$\tilde m_n(x) \coloneqq (-M) \lor \hat m_n(x) \land M $.
Write
$m(\bX) \coloneqq \bigl(m(X_1), \ldots, m(X_n)\bigr)^\T$ and
$\hat B_n(x) \coloneqq e_1^\T \hat\bH(x)^{-1} \bP_h(x)^\T \bK_h(x)
m(\bX) / n \in \R$.

\subsection{Results for Local Polynomial Estimation}

\begin{lemma}
  \label{lem:locpol_H}
  Under the set-up of Section~\ref{sec:locpol_setup},
  there exists $C \geq 1$,
  depending only on $c_K$, $C_K$, $c_X$, $C_X$, $d$ and $p$,
  such that for $h \leq 1/C$, we have
  $\sup_{x \in [0, 1]^d} \|\bH(x)\|_{\op} \leq C$
  and moreover, $\bH(x)$ is invertible
  for every $x \in [0, 1]^d$, with
  $\sup_{x \in [0, 1]^d} \|\bH(x)^{-1}\|_{\op} \leq C$.
\end{lemma}

\begin{proof}[Proof of Lemma~\ref{lem:locpol_H}]
  Throughout the proof, $C_1, C_2, \ldots > 0$ denote quantities
  depending only on $c_K$, $C_K$, $c_X$, $C_X$, $d$ and $p$.
  For $z \in [-1, 1]^d$,
  each entry of $P_h(z) K_h(z) P_h(z)^\T$
  is bounded in magnitude by $K_h(z)$, as
  $K_h(z) = 0$ unless $\|z\|_\infty \leq h$.
  Thus, for each $x \in [0, 1]^d$,
  \begin{align*}
    \|\bH(x)\|_{\op}
    &=
    \biggl\|
    \int_{[0,1]^d}
    P_h(y - x) K_h(y - x) P_h(y - x)^\T
    f_X(y)
    \,\diff y
    \biggr\|_{\op} \\
    &\leq
    D C_X
    \int_{[0,1]^d}
    \bigl\|
    P_h(y - x) K_h(y - x) P_h(y - x)^\T
    \bigr\|_{\max}
    \,\diff y \\
    &\leq
    D C_X
    \int_{[0,1]^d}
    K_h(y - x)
    \,\diff y
    \leq
    D C_X
    \int_{\R^d}
    K(z)
    \,\diff z
    \leq
    2^d C_K D C_X
    \eqqcolon C_1.
  \end{align*}
  For the second result, take $v \in \R^D$ with $\|v\|_2 = 1$.
  Since $K(x) \geq c_K$ for each $x \in [-c_K, c_K]^d$,
  and as
  either $-u/h \leq -c_K$ or $(1-u)/h \geq c_K$ for $h \leq 1/(2
  c_K)$ and $u \in [-c_K, c_K]$, we have
  for $x = (x_1,\ldots,x_d)^\T \in [0, 1]^d$ that
  \begin{align*}
    v^\T \bH(x) v
    &=
    \int_{[0,1]^d}
    K_h(y - x)
    \bigl\{v^\T P_h(y - x)\bigr\}^2
    f_X(y)
    \,\diff y \\
    &\geq
    c_X
    \int_{\prod_{j=1}^d [-x_j/h,(1-x_j)/h]}
    K(z)
    \bigl\{v^\T P_1(z)\bigr\}^2
    \,\diff z \\
    &\geq
    c_X
    c_K
    \min_{(\zeta_1,\ldots,\zeta_d) \in \{-1, 1\}^d}
    \int_{\prod_{j=1}^d (\zeta_j [0,c_K])}
    \bigl\{v^\T P_1(z)\bigr\}^2
    \,\diff z,
  \end{align*}
  where
  $(-1) \cdot [0, c_K] \coloneqq [-c_K, 0]$.
  The integrand above is a non-negative polynomial in $z$
  that is not identically zero,
  so there are finitely many points $z \in \R^d$ satisfying
  $v^\T P_1(z) = 0$.
  As the Lebesgue
  integral of an almost everywhere positive continuous function over
  a positive measure set
  is positive, we deduce that there exists
  $c(v, d, p, c_K) > 0$ such that
  $v^\T \bH(x) v \geq c_X c(v, d, p, c_K)$.
  Therefore, the map $v \mapsto v^\T \bH(x) v$
  defined for $v \in \R^D$ with $\|v\|_2 = 1$ is a continuous
  function on a compact set that is positive everywhere,
  so in particular it is bounded away from zero.
  We conclude that
  \begin{align*}
    \inf_{v \in \R^D: \, \|v\|_2 = 1}
    v^\T \bH(x) v
    &\geq
    \frac{1}{C_2},
  \end{align*}
  as required.
\end{proof}

\begin{lemma}
  \label{lem:multi_index}
  Let $d \in \N$ and $\nu \in \N_0^d$.
  For each $x, y \in \R^d$,
  \begin{align*}
    |x^\nu - y^\nu|
    &\leq
    |\nu|
    \bigl(\|x\|_\infty \lor \|y\|_\infty\bigr)^{|\nu| - 1}
    \|x - y\|_\infty.
  \end{align*}
  %
  %
\end{lemma}

\begin{proof}[Proof of Lemma~\ref{lem:multi_index}]
  Write $x = (x_1,\ldots,x_d)^\T$ and $y = (y_1,\ldots,y_d)^\T$.  By
  telescoping the sum and the mean value theorem,
  \begin{align*}
    |x^\nu - y^\nu|
    &=
    \biggl|
    \prod_{j=1}^d x_j^{\nu_j}
    - \prod_{j=1}^d y_j^{\nu_j}
    \biggr|
    =
    \biggl|
    \sum_{r=1}^d
    \biggl\{
      \biggl(
        \prod_{j=r+1}^d x_j^{\nu_j}
      \biggr)
      \biggl(
        \prod_{\ell=1}^{r-1} y_\ell^{\nu_\ell}
      \biggr)
      (x_r^{\nu_r} - y_r^{\nu_r})
    \biggr\}
    \biggr| \\
    &\leq
    \sum_{r=1}^d
    \biggl(
      \prod_{j=r+1}^d |x_j|^{\nu_j}
    \biggr)
    \biggl(
      \prod_{\ell=1}^{r-1} |y_\ell|^{\nu_\ell}
    \biggr)
    |x_r^{\nu_r} - y_r^{\nu_r}| \\
    &\leq
    \sum_{r=1}^d
    \biggl(
      \prod_{j=r+1}^d |x_j|^{\nu_j}
    \biggr)
    \biggl(
      \prod_{\ell=1}^{r-1} |y_\ell|^{\nu_\ell}
    \biggr)
    \nu_r
    \bigl(|x_r| \lor |y_r|\bigr)^{\nu_r - 1}
    |x_r - y_r| \\
    &\leq
    \sum_{r=1}^d
    \nu_r
    \|x - y\|_\infty
    \bigl(\|x\|_\infty \lor \|y\|_\infty\bigr)^{|\nu| - 1}
    =
    |\nu|
    \bigl(\|x\|_\infty \lor \|y\|_\infty\bigr)^{|\nu| - 1}
    \|x - y\|_\infty,
  \end{align*}
  as required.
\end{proof}

\begin{lemma}
  \label{lem:locpol_Hhat}
  Assume the set-up of Section~\ref{sec:locpol_setup}.
  There exists $C \geq 1$,
  depending only on $c_K$, $C_K$, $c_X$, $C_X$, $L_K$, $d$ and $p$,
  such that, for $h \leq 1/C$, we have
  with probability at least
  $1 - \exp(-n h^d/C) / h^C$ that
  $\sup_{x \in [0, 1]^d} \|\hat \bH(x)\|_{\op} \leq C$
  and moreover,
  $\hat \bH(x)$ is invertible for every
  $x \in [0, 1]^d$, satisfying
  $\sup_{x \in [0, 1]^d} \|\hat \bH(x)^{-1}\|_{\op} \leq C$.
\end{lemma}

\begin{proof}[Proof of Lemma~\ref{lem:locpol_Hhat}]
  Throughout the proof, $C_1, C_2, \ldots > 0$ denote quantities
  depending only on $c_K$, $C_K$, $c_X$, $C_X$, $L_K$, $d$ and $p$.
  For $i \in [n]$ and $j, k \in [D]$, write
  \begin{align*}
    u_{i j k}(x)
    &\coloneqq
    P_h(X_i - x)_j K_h(X_i - x) P_h(X_i - x)_k
    - \E\bigl\{ P_h(X_i - x)_j K_h(X_i - x) P_h(X_i - x)_k \bigr\},
  \end{align*}
  so that $\hat \bH(x)_{j k} - \bH(x)_{j k} = \sum_{i=1}^n u_{i j k} / n$.
  Further, $\E(u_{i j k}) = 0$ and $|u_{i j k}| \leq 2 C_K / h^d$.
  Also,
  \begin{align*}
    \E(u_{i j k}^2)
    &\leq
    \E\bigl\{
      P_h(X_i - x)_j^2 K_h(X_i - x)^2 P_h(X_i - x)_k^2
    \bigr\}
    \leq
    \E\bigl\{
      K_h(X_i - x)^2
    \bigr\} \\
    &=
    \int_{[0, 1]^d}
    K_h(y - x)^2
    f_X(y)
    \, \diff y
    \leq
    \frac{C_X}{h^d}
    \int_{\R^d}
    K(z)^2
    \, \diff z
    \leq
    \frac{2^dC_X C_K^2}{h^d}.
  \end{align*}
  Now, by Bernstein's inequality,
  for all $t > 0$,
  \begin{align*}
    \P\Biggl(
      \biggl|
      \frac{1}{n}
      \sum_{i=1}^n u_{i j k}
      \biggr|
      > \frac{C_1 t}{\sqrt{n h^d}}
      + \frac{C_1 t^2}{n h^d}
    \Biggr)
    \leq 2 e^{-t^2/2}.
  \end{align*}
  Therefore, by a union bound and as
  $\|A\|_{\op} \leq D \|A\|_{\max}$ for $A \in \R^{D \times D}$,
  \begin{align*}
    \P\biggl(
      \bigl\|\hat \bH(x) - \bH(x)\bigr\|_{\op}
      > \frac{C_1 D t}{\sqrt{n h^d}}
      + \frac{C_1 D t^2}{n h^d}
    \biggr)
    \leq D (D+1) e^{-t^2/2}.
  \end{align*}
  Therefore, by Weyl's inequality
  \citep[e.g.][Corollary~III.2.6]{bhatia1997matrix}
  and Lemma~\ref{lem:locpol_H},
  with probability at least $1 - C_2 e^{-t^2/2}$,
  where $C_2 \geq 1$,
  %
  \begin{align*}
    \lambda_{\min}\bigl(\hat\bH(x)\bigr)
    &\geq
    \lambda_{\min}\bigl(\bH(x)\bigr)
    - \bigl\|\hat \bH(x) - \bH(x)\bigr\|_{\op}
    \geq
    \frac{1}{C_2} - \frac{C_2 D t}{\sqrt{n h^d}} - \frac{C_2 D t^2}{n h^d}, \\
    \lambda_{\max}\bigl(\hat\bH(x)\bigr)
    &\leq
    \lambda_{\max}\bigl(\bH(x)\bigr)
    + \bigl\|\hat \bH(x) - \bH(x)\bigr\|_{\op}
    \leq
    C_2 + \frac{C_2 D t}{\sqrt{n h^d}} + \frac{C_2 D t^2}{n h^d}.
  \end{align*}
  Setting $t \coloneqq \sqrt{n h^d} / (4 C_2^2 D)$, we see
  that with probability at least
  $1 - C_3 \exp(-n h^d / C_3)$,
  %
  \begin{align*}
    \lambda_{\min}\bigl(\hat\bH(x)\bigr)
    &\geq
    \frac{1}{C_2}
    - \frac{1}{4 C_2}
    - \frac{1}{16 C_2^3 D}
    \geq \frac{1}{C_3}, \\
    \lambda_{\max}\bigl(\hat\bH(x)\bigr)
    &\leq
    C_2
    + \frac{1}{4 C_2}
    + \frac{1}{16 C_2^3 D}
    \leq C_3.
  \end{align*}
  Observe that $K_h(\cdot)$ is
  bounded by
  $C_K/h^d$ and is
  $L_K/h^{d+1}$-Lipschitz
  on $\R^d$ with respect to $\|\cdot\|_\infty$.
  Further,
  for $j \in [D]$, the restriction of
  $P_h(\cdot)_j$ to $[-h,h]^d$ is
  bounded by $1$ and is
  $p/h$-Lipschitz
  with respect to $\|\cdot\|_\infty$,
  by Lemma~\ref{lem:multi_index}.
  Thus, for $x, y, z \in [0, 1]^d$
  and $j, k \in [D]$,
  \begin{align*}
    \bigl|
    P_h(z-y)_j K_h(z-y) P_h(z-y)_k
    &- P_h(z-x)_j K_h(z-x) P_h(z-x)_k
    \bigr|
    \iftoggle{journal}{}{\\ &}
    \leq
    \frac{L_K + 2 p C_K}{h^{d+1}}
    \|y - x\|_\infty
    \leq
    \frac{C_4}{h^{d+1}}
    \|y - x\|_\infty.
  \end{align*}
  %
  We deduce that $x \mapsto \hat\bH(x)$
  is $C_4 / h^{d+1}$-Lipschitz from $\|\cdot\|_\infty$ to
  $\|\cdot\|_{\max}$,
  where we may take $C_4 \geq C_3 \lor 1$,
  so it is also $C_4 D / h^{d+1}$-Lipschitz from
  $\|\cdot\|_\infty$ to $\|\cdot\|_{\op}$.
  It follows by Weyl's inequality that both
  $x \mapsto \lambda_{\min}\bigl(\hat\bH(x)\bigr)$
  and $x \mapsto \lambda_{\max}\bigl(\hat\bH(x)\bigr)$
  are $C_4 D / h^{d+1}$-Lipschitz with respect to $\|\cdot\|_\infty$.
  Let $\delta \coloneqq h^{d+1} / (2 C_4^2 D) \leq 1$ and
  let $\cX_h$ denote a $\delta$-cover of $[0, 1]^d$ with respect to
  $\|\cdot\|_\infty$ of
  cardinality at most $(2/\delta)^d$.
  By a union bound, with probability at least
  $1 - C_3 (2/\delta)^d \exp(-n h^d / C_3)$,
  \begin{align*}
    \inf_{x \in [0, 1]^d}
    \lambda_{\min} \bigl(\hat \bH(x)\bigr)
    &\geq
    \min_{x \in \cX_h}
    \lambda_{\min} \bigl(\hat \bH(x)\bigr)
    - \frac{1}{2 C_4}
    \geq
    \frac{1}{C_3} - \frac{1}{2 C_4}
    \geq
    \frac{1}{2 C_4}.
  \end{align*}
  Applying the same logic to the maximum eigenvalue,
  we see that both
  $\inf_{x \in [0, 1]^d} \lambda_{\min} \bigl(\hat \bH(x)\bigr) \geq 1 / C_5$
  and $\sup_{x \in [0, 1]^d} \lambda_{\max} \bigl(\hat \bH(x)\bigr)
  \leq C_5$,
  with probability at least
  $1 - \exp(-n h^d/C_5) / h^{C_5}$
  for $h \leq 1/C_5$.
\end{proof}

\begin{lemma}
  \label{lem:locpol_stochastic}
  Assume the set-up of Section~\ref{sec:locpol_setup}.
  There exists $C > 0$,
  depending only on $c_K$, $C_K$, $c_X$, $C_X$, $L_K$, $d$ and $p$,
  such that for all $h \leq 1/C$ and $t > 0$, we have with probability at least
  $1 - C \exp(-t^2) - \exp(-n h^d/C) / h^C$ that
  \begin{equation*}
    \int_{[0, 1]^d}
    \big(
      \hat m_n(x) - \hat B_n(x)
    \bigr)^2
    \,\diff x
    \leq
    \frac{C t^2}{n h^d}.
  \end{equation*}
\end{lemma}

\begin{proof}[Proof of Lemma~\ref{lem:locpol_stochastic}]
  Throughout the proof, $C_1, C_2, \ldots > 0$ denote quantities
  depending only on $c_K$, $C_K$, $c_X$, $C_X$, $L_K$, $d$ and $p$.
  With $\varepsilon \coloneqq (\varepsilon_1, \ldots,
  \varepsilon_n)^\T$, we have
  %
  \begin{align*}
    \int_{[0, 1]^d}
    \big(
      \hat m_n(x) - \hat B_n(x)
    \bigr)^2
    \,\diff x
    &=
    \frac{1}{n^2}
    \int_{[0, 1]^d}
    \bigl(
      e_1^\T
      \hat\bH(x)^{-1}
      \bP_h(x)^\T \bK_h(x) \varepsilon
    \bigr)^2
    \,\diff x \\
    &\leq
    \frac{1}{n^2}
    \sup_{x \in [0, 1]^d}
    \bigl\|\hat\bH(x)^{-1}\bigr\|_\op^2
    \int_{[0, 1]^d}
    \bigl\|
    \bP_h(x)^\T \bK_h(x) \varepsilon
    \bigr\|_2^2
    \,\diff x.
  \end{align*}
  For $i, j \in [n]$, define the random variables
  \begin{align*}
    U_{i j}
    \coloneqq
    \sum_{k=1}^D
    \int_{[0, 1]^d}
    P_h(X_i - x)_k P_h(X_j - x)_k K_h(X_i - x) K_h(X_j - x)
    \varepsilon_i \varepsilon_j
    \,\diff x,
  \end{align*}
  so that, for each $i, j \in [n]$, we have
  $\E(U_{i j} \mid X_i, \varepsilon_i) = 0$ and
  \begin{align*}
    \int_{[0, 1]^d}
    \bigl\|
    \bP_h(x)^\T \bK_h(x) \varepsilon
    \bigr\|_2^2
    \,\diff x
    &=
    \int_{[0, 1]^d}
    \biggl\|
    \sum_{i=1}^n
    P_h(X_i - x) K_h(X_i - x) \varepsilon_i
    \biggr\|_2^2
    \,\diff x
    =
    \sum_{i=1}^n
    \sum_{j=1}^n
    U_{i j}.
  \end{align*}
  By the Cauchy--Schwarz inequality, for $i, j \in [n]$, and
  since $|\varepsilon_i| \leq 1$,
  \begin{align*}
    |U_{i j}|
    &\leq
    \max_{r \in [n]}
    \sum_{k=1}^D
    \int_{[0, 1]^d}
    P_h(X_r - x)_k^2 K_h(X_r - x)^2
    \,\diff x
    \leq
    D
    \sup_{y \in [0, 1]^d}
    \int_{[0, 1]^d}
    K_h(y - x)^2
    \,\diff x \\
    &=
    \frac{D}{h^{2d}}
    \sup_{y \in [0, 1]^d}
    \int_{[0, 1]^d}
    K\biggl(\frac{y - x}{h}\biggr)^2
    \,\diff x
    \leq
    \frac{D}{h^{d}}
    \int_{\R^d}
    K(u)^2
    \,\diff u
    \leq
    \frac{C_K^2 2^d D}{h^{d}}
    \leq \frac{C_2}{h^d}.
  \end{align*}
  In particular,
  $\sum_{i=1}^n U_{i i} \leq C_2 n / h^d$.
  Moreover, by Hoeffding's inequality
  for degenerate second-order $U$-statistics
  \citep[e.g.,][Theorem~4.1.12b]{de1999decoupling},
  for all $t > 0$,
  \begin{align*}
    \P\biggl(
      \biggl|
      \sum_{i=1}^n
      \sum_{j \in [n] \setminus \{i\}}
      U_{i j}
      \biggr|
      > \frac{C_3 n t^2}{h^d}
    \biggr)
    \leq
    C_3 e^{-t^2}.
  \end{align*}
  Therefore, with probability at least
  $1 - C_3 e^{-t^2}$, where $C_3 \geq e$,
  \begin{align*}
    \int_{[0, 1]^d}
    \bigl\|
    \bP_h(x)^\T \bK_h(x) \varepsilon
    \bigr\|_2^2
    \,\diff x
    &\leq
    \frac{C_2 n}{h^d}
    + \frac{C_3 n t^2}{h^d}
    \leq
    \frac{C_4 n t^2}{h^d},
  \end{align*}
  as there is nothing to prove if $t \in [0, 1]$.
  By Lemma~\ref{lem:locpol_Hhat} and a union bound,
  we conclude that with probability at least
  $1 - C_5 \exp(-t^2) - \exp(-n h^d/C_5) / h^{C_5}$,
  we have
  \begin{align*}
    \int_{[0, 1]^d}
    \big(
      \hat m_n(x) - \hat B_n(x)
    \bigr)^2
    \,\diff x
    &\leq
    \frac{C_5 t^2}{n h^d}.
    \qedhere
  \end{align*}
\end{proof}

\begin{lemma}
  \label{lem:locpol_bias}
  Assume the set-up of Section~\ref{sec:locpol_setup}.
  There exists $C \geq 1$,
  depending only on $c_K$, $C_K$, $c_X$, $C_X$, $L_K$, $d$, $p$ and $M$
  such that for $h \leq 1/C$, we have
  with probability at least
  $1 - \exp(-n h^d/C) / h^C$ that
  \begin{equation*}
    \int_{[0, 1]^d}
    \big(
      \hat B_n(x) - m(x)
    \bigr)^2
    \,\diff x
    \leq
    C h^{2 \alpha}.
  \end{equation*}
\end{lemma}

\begin{proof}[Proof of Lemma~\ref{lem:locpol_bias}]
  Throughout the proof, $C_1, C_2, \ldots > 0$ denote quantities
  depending only on $c_K$, $C_K$, $c_X$, $C_X$, $L_K$, $d$, $p$ and $M$.
  By Taylor's theorem,
  as $m \in \cH(d,\alpha, M)$,
  for all $x, x' \in [0, 1]^d$,
  there exists $\tilde x$
  on the line segment between $x$ and $x'$ such that
  \begin{align}
    \label{eq:locpol_bias}
    m(x') - m(x)
    &=
    \!\!\!\!
    \sum_{\nu \in \N_0^d: |\nu| \in [\underline\alpha]}
    \!\!\!\!
    \frac{(x' - x)^\nu}{\nu!}
    \partial_\nu m(x)
    +
    \!\!\!\!
    \sum_{\nu \in \N_0^d: |\nu| = \underline\alpha}
    \!\!\!\!
    \frac{(x' - x)^\nu}{\nu!}
    \bigl\{
      \partial_\nu m(\tilde x)
      - \partial_\nu m(x)
    \bigr\}.
  \end{align}
  Since $p \geq \underline\alpha$,
  the first term on the right-hand side of \eqref{eq:locpol_bias} satisfies
  \begin{align}
    \label{eq:locpol_bias_01}
    \sum_{\nu \in \N_0^d: |\nu| \in [\underline\alpha]}
    \frac{(x' - x)^\nu}{\nu!}
    \partial_\nu m(x)
    &=
    \!
    \sum_{\nu \in \N_0^d: |\nu| \in [\underline\alpha]}
    h^{|\nu|} P_h(x' - x)^\T e_\nu
    \partial_\nu m(x).
  \end{align}
  For the second term on the right-hand side of \eqref{eq:locpol_bias},
  for $\|x' - x\|_\infty \leq h$, we have
  \begin{align}
    \label{eq:locpol_bias_02}
    \biggl|
    \sum_{\nu \in \N_0^d: |\nu| = \underline\alpha}
    \!\!
    \frac{(x' - x)^\nu}{\nu!}
    \bigl\{
      \partial_\nu m(\tilde x)
      - \partial_\nu m(x)
    \bigr\}
    \biggr|
    &\leq
    \sum_{\nu \in \N_0^d: |\nu| = \underline\alpha}
    \!\!
    h^{\underline\alpha}
    M \|\tilde x - x\|_2^{\alpha - \underline\alpha}
    \leq
    C_1
    h^{\alpha}.
  \end{align}
  Note that $1_n \coloneqq (1, \ldots, 1)^\T = \bP_h(x) e_1 \in \R^n$, so by
  \eqref{eq:locpol_bias}, \eqref{eq:locpol_bias_01}
  and \eqref{eq:locpol_bias_02},
  \begin{align*}
    \bigl|
    \hat B_n(x)
    \iftoggle{journal}{}{&}
    - m(x)
    \bigr|
    \iftoggle{journal}{&}{}
    =
    \bigl|
    e_1^\T
    \hat\bH(x)^{-1}
    \bP_h(x)^\T \bK_h(x) m(\bX) / n
    - m(x)
    \bigr| \\
    &=
    \frac{1}{n}
    \bigl|
    e_1^\T
    \hat\bH(x)^{-1}
    \bP_h(x)^\T \bK_h(x)
    \bigl\{ m(\bX) - m(x) 1_n \bigr\}
    \bigr| \\
    &=
    \frac{1}{n}
    \biggl|
    e_1^\T
    \hat\bH(x)^{-1}
    \sum_{i=1}^n
    P_h(X_i - x) K_h(X_i - x)
    \bigl\{ m(X_i) - m(x) \bigr\}
    \biggr| \\
    &\leq
    \iftoggle{journal}{}{\!\!\!\!}
    \sum_{\nu \in \N_0^d: |\nu| \in [\underline\alpha]}
    \iftoggle{journal}{}{\!\!\!}
    \bigl|
    h^{|\nu|}
    e_1^\T
    e_\nu
    \partial_\nu m(x)
    \bigr|
    + \frac{C_1 h^\alpha}{n}
    \bigl\|\hat\bH(x)^{-1}\bigr\|_\op
    \biggl\|\sum_{i=1}^n \|P_h(X_i-x)\|_\infty 1_D
    K_h(X_i-x) \biggr\|_2
    \\
    &\leq
    \frac{C_1 \sqrt{D} h^\alpha}{n}
    \bigl\|\hat\bH(x)^{-1}\bigr\|_\op
    \sum_{i=1}^n
    K_h(X_i - x).
  \end{align*}
  Therefore,
  \begin{align*}
    \int_{[0, 1]^d}
    \bigl(
      \hat B_n(x) - m(x)
    \bigr)^2
    \,\diff x
    &\leq
    \frac{C_1^2 D h^{2\alpha}}{n^2}
    \sup_{x \in [0, 1]^d}
    \bigl\|\hat\bH(x)^{-1}\bigr\|_\op^2
    \int_{[0, 1]^d}
    \biggl(
      \sum_{i=1}^n
      K_h(X_i - x)
    \biggr)^2
    \,\diff x.
  \end{align*}
  For $i \in [n]$ and $x \in [0, 1]^d$, let
  $u(x) \coloneqq \E\bigl\{K_h(X_1 - x)\bigr\}$ and
  $U_{i}(x) \coloneqq K_h(X_i - x) - u(x)$, so
  \begin{align}
    \nonumber
    &\int_{[0, 1]^d}
    \biggl(
      \sum_{i=1}^n
      K_h(X_i - x)
    \biggr)^2
    \,\diff x
    =
    \sum_{i=1}^n
    \sum_{j=1}^n
    \int_{[0, 1]^d}
    K_h(X_i - x)
    K_h(X_j - x)
    \,\diff x \\
    \nonumber
    &\quad=
    \sum_{i=1}^n
    \sum_{j=1}^n
    \int_{[0, 1]^d}
    \bigl\{U_i(x) + u(x)\bigr\}
    \bigl\{U_j(x) + u(x)\bigr\}
    \,\diff x \\
    \label{eq:bias_new_1}
    &\quad=
    \sum_{i=1}^n
    \sum_{j=1}^n
    \int_{[0, 1]^d}
    U_i(x) U_j(x)
    \,\diff x
    + 2 n
    \sum_{i=1}^n
    \int_{[0, 1]^d}
    u(x)
    U_i(x)
    \,\diff x
    + n^2 \int_{[0, 1]^d}
    u(x)^2
    \,\diff x.
  \end{align}
  Now, $0 \leq u(x) \leq 2^d C_K C_X$, so
  \begin{align*}
    \int_{[0, 1]^d}
    u(x)^2
    \,\diff x
    &\leq
    2^{2 d} C_K^2 C_X^2
    \leq C_2.
  \end{align*}
  %
  Next, for $i \in [n]$,
  \begin{align*}
    \biggl|
    \int_{[0, 1]^d} u(x) U_i(x) \,\diff x
    \biggr|
    &\leq
    2^d C_K C_X
    \int_{[0, 1]^d}
    \bigl|K_h(X_i - x) - u(x)\bigr|
    \,\diff x
    \leq
    2^{2 d} C_K^2 C_X
    + 2^{2d} C_K^2 C_X^2
    \leq C_3.
  \end{align*}
  Therefore, by Hoeffding's inequality,
  for each $t > 0$, we have with probability at least
  $1 - e^{-t^2}$ that
  \begin{align*}
    \sum_{i=1}^n
    \int_{[0, 1]^d}
    u(x)
    U_i(x)
    \,\diff x
    \leq
    C_3 t \sqrt{2n}.
  \end{align*}
  Since $|U_i(x)| \leq C_K h^{-d}$
  for each $i \in [n]$, we have
  for $i, j \in [n]$ that
  \begin{align*}
    \biggl|
    \int_{[0, 1]^d}
    U_i(x) U_j(x)
    \,\diff x
    \biggr|
    &\leq
    C_K h^{-d}
    \int_{[0, 1]^d}
    \bigl|K_h(X_i - x) - u(x)\bigr|
    \,\diff x
    \leq
    C_4 h^{-d}.
  \end{align*}
  In particular,
  $\sum_{i=1}^n \int_{[0, 1]^d} U_i(x)^2 \,\diff x \leq C_4 n h^{-d}$.
  Further, by Hoeffding's inequality
  for degenerate second-order $U$-statistics
  \citep[e.g.,][Theorem~4.1.12b]{de1999decoupling},
  for all $t > 0$,
  with probability at least $1 - C_5 e^{-t^2}$,
  %
  \begin{align*}
    \sum_{i=1}^n
    \sum_{j=1}^n
    \int_{[0, 1]^d}
    U_i(x) U_j(x)
    \,\diff x
    \leq
    C_5 n h^{-d} t^2.
  \end{align*}
  Combining the bounds on terms in \eqref{eq:bias_new_1},
  we deduce that,
  with probability at least $1 - C_6 e^{-t^2}$,
  \begin{align*}
    \int_{[0, 1]^d}
    \biggl(
      \sum_{i=1}^n
      K_h(X_i - x)
    \biggr)^2
    \,\diff x
    &\leq
    C_6 n h^{-d} t^2
    + C_6 t n^{3/2}
    + C_6 n^2.
  \end{align*}
  Taking $t \coloneqq \sqrt{n h^d}$ and since $h \leq 1$,
  we obtain that
  with probability at least $1 - C_6 \exp(-n h^d)$,
  \begin{align*}
    \int_{[0, 1]^d}
    \biggl(
      \sum_{i=1}^n
      K_h(X_i - x)
    \biggr)^2
    \,\diff x
    &\leq
    3 C_6 n^2.
  \end{align*}
  Therefore, by Lemma~\ref{lem:locpol_Hhat}, with probability at least
  $1 - \exp(-n h^d / C_7) / h^{C_7}$,
  \begin{align*}
    \int_{[0, 1]^d}
    \bigl(
      \hat B_n(x) - m(x)
    \bigr)^2
    \,\diff x
    &\leq
    C_7 h^{2\alpha},
  \end{align*}
  as required.
\end{proof}

\begin{theorem}
  \label{thm:locpol}
  Assume the set-up of Section~\ref{sec:locpol_setup}.
  There exists $C \geq 1$
  depending only on $c_K$, $C_K$, $c_X$, $C_X$, $L_K$, $d$, $p$ and $M$
  such that, for $h \leq 1/C$ and all $t > 0$,
  with probability at least
  $1 - C \exp(-t^2) - \exp(-n h^d/C) / h^C$,
  \begin{equation*}
    \int_{[0, 1]^d}
    \big(
      \hat m_n(x) - m(x)
    \bigr)^2
    \,\diff x
    \leq
    C h^{2 \alpha}
    + \frac{C t^2}{n h^d}.
  \end{equation*}
  In particular, if $h = n^{-1/(2 \alpha + d)}$, then with probability
  at least $1 - Ce^{-t^2}
  - n^C e^{-n^{2\alpha / (2\alpha+d)} / C}$,
  \begin{align*}
    \int_{[0, 1]^d}
    \big(
      \hat m_n(x) - m(x)
    \bigr)^2
    \,\diff x
    &\leq
    C t^2 n^{-\frac{2\alpha}{2\alpha + d}}.
  \end{align*}
  Moreover, under the same condition on $h$,
  the truncated local polynomial estimator satisfies
  \begin{align*}
    \int_{[0, 1]^d}
    \E\Bigl\{
      \big(
        \tilde m_n(x) - m(x)
      \bigr)^2
    \Bigr\}
    \,\diff x
    &\leq
    C n^{-\frac{2\alpha}{2\alpha + d}}.
  \end{align*}
\end{theorem}

\begin{proof}[Proof of Theorem~\ref{thm:locpol}]
  Throughout the proof, $C_1, C_2, \ldots > 0$ are quantities
  depending only on $c_K$, $C_K$, $c_X$, $C_X$, $L_K$, $d$, $p$ and $M$.
  For the first result,
  by Lemmas~\ref{lem:locpol_stochastic}
  and~\ref{lem:locpol_bias},
  with probability at least
  $1 - C_1 \exp(-t^2) - \exp(-n h^d/C_1) / h^{C_1}$,
  \begin{align*}
    \int_{[0, 1]^d}
    \!
    \big(
      \hat m_n(x) - m(x)
    \bigr)^2
    \,\diff x
    &\leq
    2
    \int_{[0, 1]^d}
    \!
    \big(
      \hat m_n(x) - \hat B_n(x)
    \bigr)^2
    \,\diff x
    + 2 \int_{[0, 1]^d}
    \!
    \big(
      \hat B_n(x) - m(x)
    \bigr)^2
    \,\diff x
    \iftoggle{journal}{}{\\ &}
    \leq
    \frac{C_1 t^2}{n h^d}
    + C_1 h^{2 \alpha}.
  \end{align*}
  %
  For the second bound,
  taking $h = n^{-1/(2 \alpha + d)}$ yields
  \begin{align*}
    \int_{[0, 1]^d}
    \big(
      \hat m_n(x) - m(x)
    \bigr)^2
    \,\diff x
    \leq
    C_1 (t^2 + 1)
    n^{-\frac{2 \alpha}{2 \alpha + d}}
    \leq
    C_2 t^2 n^{-\frac{2 \alpha}{2 \alpha + d}},
  \end{align*}
  with probability at least
  $1 - C_1 \exp(-t^2) - n^{C_1}
  \exp(-n^{2\alpha/(2\alpha + d)} / C_1)$,
  as taking $C_1 \geq e$ makes
  this probability trivial unless $t \geq 1$.
  For the third inequality, by Fubini's theorem and
  integrating the tail probability, as
  $\sup_{x \in [0, 1]^d} |\tilde m_n(x) - m(x)| \leq 2 M$,
  \begin{align*}
    &\int_{[0, 1]^d}
    \E\Bigl\{
      \big(
        \tilde m_n(x) - m(x)
      \bigr)^2
    \Bigr\}
    \,\diff x
    =
    C_2 n^{\frac{-2\alpha}{2\alpha + d}}
    \,
    \E\Biggl\{
      \frac{n^{\frac{2\alpha}{2\alpha + d}}}{C_2}
      \int_{[0, 1]^d}
      \big(
        \tilde m_n(x) - m(x)
      \bigr)^2
      \,\diff x
    \Biggr\} \\
    &\qquad=
    C_2 n^{\frac{-2\alpha}{2\alpha + d}}
    \int_0^\infty
    \P \biggl(
      \int_{[0, 1]^d}
      \big(
        \tilde m_n(x) - m(x)
      \bigr)^2
      \,\diff x
      > C_2 s n^{-\frac{2\alpha}{2\alpha + d}}
    \biggr)
    \,\diff s \\
    &\qquad\leq
    C_2 n^{\frac{-2\alpha}{2\alpha + d}}
    \int_0^{4 M^2 n^{\frac{2\alpha}{2\alpha + d}} / C_2}
    \Bigl\{
      C_1 \exp(-s)
      + n^{C_1}
      \exp\bigl(-n^{2\alpha/(2\alpha + d)} / C_1\bigr)
    \Bigr\}
    \,\diff s \\
    &\qquad\leq
    C_1 C_2 n^{\frac{-2\alpha}{2\alpha + d}}
    + 4 M^2 n^{C_1}
    \exp\bigl(-n^{2\alpha/(2\alpha + d)} / C_1\bigr).
  \end{align*}
  Note that
  $C_1 (C_1 + 1) \log n - n^{2\alpha/(2\alpha+d)}
  \to -\infty$
  as $n \to \infty$.
  Thus, there exists $C_3 > 0$ such that
  for $n \geq C_3$, we have
  $n^{C_1} \exp\bigl(-n^{2\alpha/(2\alpha + d)} / C_1\bigr) \leq 1/n$.
  We deduce that for $n \geq C_3$,
  \begin{align*}
    \int_{[0, 1]^d}
    \E\Bigl\{
      \big(
        \tilde m_n(x) - m(x)
      \bigr)^2
    \Bigr\}
    \,\diff x
    &\leq
    C_1 C_2
    n^{\frac{-2\alpha}{2\alpha + d}}
    +
    \frac{4 M^2}{n}
    \leq
    C_4 n^{\frac{-2\alpha}{2\alpha + d}}.
  \end{align*}
  As $\sup_{x \in [0, 1]^d}|\tilde m_n(x) - m(x)| \leq 2 M$,
  this holds for all $n \in \N$
  after increasing $C_4$ to $C_5$ if necessary.
\end{proof}

\section{Approximation Theory}

We construct an explicit transformer
that produces outputs similar to those
of truncated local polynomial estimation,
keeping track of its
architecture and parameter magnitudes.

\subsection{ReLU Neural Networks}
We begin by summarising some approximation properties of ReLU neural
networks that will be useful later for our transformer construction.
\begin{definition}
  Let $\din, \dout, N, L \in \N$ and $B>0$.
  A function $f : \R^{\din} \to
  \R^{\dout}$ is a \emph{(ReLU) neural network}
  with width $N$, depth $L$ and all parameters bounded by $B$ if
  there exist $\bW_\ell \in [-B,B]^{ d_{\ell} \times d_{\ell-1}}$
  and $b_{\ell} \in [-B,B]^{d_{\ell}}$ for $\ell\in[L+1]$, where
  $d_0\coloneqq \din$, $d_{L+1} \coloneqq \dout$
  and $d_{\ell} \in [N]$ for $\ell\in[L]$, such that
  \begin{align*}
    f(\cdot) = A_{L+1} \circ \ReLU \circ A_L \circ \ReLU \circ \cdots
    \circ A_2 \circ \ReLU \circ A_1 (\cdot),
  \end{align*}
  where $A_{\ell}(z) \coloneqq \bW_\ell z + b_{\ell}$ for $\ell\in[L+1]$.
\end{definition}
We may assume without loss of generality that $d_{\ell} = N$ for
$\ell\in[L]$ by padding the weight matrices $\bW_{\ell}$ and bias
vectors $b_{\ell}$ with zeros.


\begin{lemma}[Network composition] \label{lemma:network-composition}
  Suppose that for $r\in\{1,2\}$,
  \begin{align*}
    f^{(r)}(\cdot) \coloneqq A_{L^{(r)}+1}^{(r)} \circ \ReLU \circ
    \cdots \circ A_2^{(r)} \circ \ReLU \circ A_1^{(r)} (\cdot),
  \end{align*}
  where $A_{\ell}^{(r)}(z) \coloneqq \bW_\ell^{(r)} z +
  b_{\ell}^{(r)}$, $\bW_\ell^{(r)} \in
  [-B^{(r)},B^{(r)}]^{d_{\ell}^{(r)} \times d_{\ell-1}^{(r)}}$ and
  $b_{\ell}^{(r)} \in [-B^{(r)},B^{(r)}]^{d_{\ell}^{(r)}}$ for
  $\ell\in[L^{(r)}+1]$. Suppose further that the output dimension of
  $f^{(1)}$ is equal to the input dimension of $f^{(2)}$,
  i.e.~$d_{L^{(1)}+1}^{(1)} = d_0^{(2)} \eqqcolon m$. Then $f^{(2)}
  \circ f^{(1)}$ is a neural network with width
  $\max_{\ell\in[L^{(1)}]}d_{\ell}^{(1)} \vee
  \max_{\ell\in[L^{(2)}]}d_{\ell}^{(2)}$, depth $L^{(1)}+L^{(2)}$ and
  all parameters bounded by
  $\bigl\{m\|\bW_1^{(2)}\|_{\max}\bigl(\|\bW_{L^{(1)}+1}^{(1)}\|_{\max} \vee
    \|b_{L^{(1)}+1}^{(1)}\|_{\infty}\bigr) +
  \|b_1^{(2)}\|_{\infty}\bigr\} \vee B^{(1)} \vee B^{(2)}$.
\end{lemma}
\begin{proof}
  Since
  \begin{align*}
    \bW_1^{(2)}\bigl(\bW_{L^{(1)}+1}^{(1)} z +
    b_{L^{(1)}+1}^{(1)}\bigr) + b_1^{(2)} = \bW_1^{(2)}
    \bW_{L^{(1)}+1}^{(1)} z + \bigl(\bW_1^{(2)} b_{L^{(1)}+1}^{(1)} +
    b_1^{(2)}\bigr),
  \end{align*}
  the result follows.
\end{proof}

\begin{lemma} \label{lemma:ignore_residual}
  Let $d'\in\N$, $B\geq 1$,
  $\bW_1, \bW_2 \in [-B,B]^{d' \times d'}$ and $b_1,b_2\in[-B,B]^{d'}$.
  There exist $\bW_1'\in [-B, B]^{3d' \times d'}$,
  $\bW_2' \in [-B, B]^{d' \times 3d'}$,
  $b_1' \in [-B, B]^{3d'}$ and
  $b_2' \in [-B, B]^{d'}$ such that,
  for all $x\in\R^{d'}$,
  \begin{align*}
    x + \bW_2' \, \ReLU(\bW_1' x + b_1') + b_2'
    &= \bW_2 \ReLU(\bW_1 x + b_1) + b_2.
  \end{align*}
\end{lemma}

\begin{proof}
  Take $b_1' \coloneqq (b_1^\T, 0_{2d'}^\T)^\T$
  and $b_2' \coloneqq b_2$, and let
  \begin{align*}
    \bW_1'
    &\coloneqq
    \begin{pmatrix}
      \bW_1 \\ \bI_{d' \times d'} \\ -\bI_{d' \times d'}
    \end{pmatrix},
    &\bW_2'
    &\coloneqq
    \begin{pmatrix}
      \bW_2 & -\bI_{d' \times d'} & \bI_{d' \times d'}
    \end{pmatrix}.
  \end{align*}
  Since $\ReLU(-x) - \ReLU(x) = -x$ for all $x \in \R^{d'}$,
  we deduce that
  \begin{align*}
    x + \bW_2' \, \ReLU(\bW_1' x + b_1') + b_2' &= x +
    \begin{pmatrix}
      \bW_2 & -\bI_{d' \times d'} & \bI_{d' \times d'}
    \end{pmatrix}
    \begin{pmatrix}
      \ReLU(\bW_1 x + b_1) \\ \ReLU(x) \\ \ReLU(-x)
    \end{pmatrix} + b_2 \\
    &= \bW_2 \, \ReLU(\bW_1 x + b_1) + b_2,
  \end{align*}
  as required.
\end{proof}

\begin{lemma}\label{lemma:NN-multiplication}
  Let $C\geq 1$ and $N,L\in\N$. There exists a ReLU neural
  network $\phi:\R^2 \to \R$ with width $9N+1$, depth $L$ and
  all parameters bounded by $32C^2N$ such that for all $(x,y)\in[-C,C]^2$,
  \begin{align*}
    |\phi(x,y) - x y| \leq 24C^2N^{-L}.
  \end{align*}
\end{lemma}
\begin{proof}
  The result follows from \citet[Lemma~4.2]{LuShenYangZhang2020}, where
  the upper bound on the magnitude of the parameters follows by inspecting the
  proofs of their Lemmas~5.1,~5.2 and~4.2.
\end{proof}

The following lemma is an analogue of \citet[Lemma~5.3]{LuShenYangZhang2020},
but we extend the input domain from $[0,1]^k$ to $[-C,C]^k$ and track
the magnitude
of the parameters.
\begin{lemma}\label{lemma:NN-multiplication-general}
  Let $C\geq 1$, $k\geq 2$ and $N,L\in\N$. There exists a
  ReLU neural network
  $\phi:\R^k \to \R$ with width $9(N+1)+2k-1$,
  depth $7kL(k-1)$ and all parameters bounded by $3C^k(40N+40)^{2}$
  such that for all $(x_1,\ldots,x_k)\in[-C,C]^k$,
  \begin{align*}
    |\phi(x_1,\ldots,x_k) - x_1\cdots x_k| \leq 30C^k(k-1)(N+1)^{-7kL}.
  \end{align*}
\end{lemma}
\begin{proof}
  We first assume that $C=1$.
  By Lemma~\ref{lemma:NN-multiplication}, there exists a ReLU neural network
  $\phi_2:\R^2 \to \R$ with width $9(N+1)+1$,
  depth $7kL$ and all parameters bounded by $40N+40$ such that for all
  $(x,y)\in[-1.1,1.1]^2$,
  \begin{align}
    |\phi_2(x,y) - x y| \leq 30(N+1)^{-7kL}. \label{eq:induction-k=2}
  \end{align}
  Now suppose that for some $m\in\{2,\ldots,k-1\}$, there exists
  $\phi_m: \R^m \to \R$ with width $9(N+1)+2m-1$,
  depth $7kL(m-1)$ and all parameters bounded by $3(40N+40)^2$,
  such that
  \begin{align*}
    |\phi_m(x_1,\ldots,x_m) - x_1\cdots x_m| \leq 30(m-1)(N+1)^{-7kL}.
  \end{align*}
  Here, the case $m=2$ is proved by~\eqref{eq:induction-k=2}.
  We then define
  \begin{align*}
    \phi_{m+1}(x_1,\ldots,x_{m+1})\coloneqq
    \phi_2 \bigl(\phi_m(x_1,\ldots,x_m),x_{m+1}\bigr).
  \end{align*}
  Since $\ReLU(x_{m+1}) - \ReLU(-x_{m+1})=x_{m+1}$,
  the identity function $x_{m+1} \mapsto x_{m+1}$ can be implemented by a
  ReLU neural network with width $2$, any depth,
  and all parameters bounded by $1$.
  Hence, by network composition (Lemma~\ref{lemma:network-composition}),
  $\phi_{m+1}:\R^{m+1} \to \R$
  can be implemented by a ReLU neural network with width
  $9(N+1)+2(m+1)-1$, depth $7kLm$ and all parameters bounded by
  $3(40N+40)^{2}$.
  Moreover, since $30(m-1)(N+1)^{-7kL} \leq 30(k-1)2^{-7k} \leq 0.1$,
  we have $\phi_m(x_1,\ldots,x_m) \in [-1.1,1.1]$.
  Thus, by~\eqref{eq:induction-k=2},
  \begin{align*}
    &\bigl|
    \phi_{m+1}(x_1,\ldots,x_{m+1}) - x_1\cdots x_{m+1}
    \bigr|
    = \bigl|
    \phi_2\bigl(\phi_m(x_1,\ldots,x_m),x_{m+1}\bigr) - x_1\cdots x_{m+1}
    \bigr|\\
    &\qquad\leq
    \bigl|
    \phi_2\bigl(\phi_m(x_1,\ldots,x_m),x_{m+1}\bigr)
    - \phi_m(x_1,\ldots,x_m)\cdot x_{m+1}
    \bigr|
    \iftoggle{journal}{}{\\&\qquad\qquad}
    + \bigl|\phi_m(x_1,\ldots,x_m) - x_1\cdots x_{m}\bigr|
    \cdot |x_{m+1}|\\
    &\qquad\leq
    30(N+1)^{-7kL} + 30(m-1)(N+1)^{-7kL} = 30m(N+1)^{-7kL}.
  \end{align*}
  The claim for $C=1$ thus follows from induction. Now for any $C\geq
  1$, we have that $(x_1,\ldots,x_k) \mapsto
  C^k\phi_k(x_1/C,\ldots,x_k/C)$ is a neural network with width $9(N+1)+2k-1$,
  depth $7kL(k-1)$ and all parameters bounded by $3C^k(40N+40)^{2}$. Moreover,
  \begin{align*}
    \bigl|C^k\phi_k(x_1/C,\ldots,x_k/C) - x_1\cdots x_k \bigr| &= C^k
    \bigl|\phi_k(x_1/C,\ldots,x_k/C) - (x_1/C)\cdots (x_k/C) \bigr|\\
    &\leq 30C^k(k-1)(N+1)^{-7kL},
  \end{align*}
  for all $(x_1,\ldots,x_k) \in [-C,C]^k$.
\end{proof}

\begin{lemma}\label{lemma:NN_polynomial}
  Let $C\geq 1$, $d,k,N,L\in\N$ and $\nu =
  (\nu_1,\ldots,\nu_d) \in \N_0^d$ be such that $|\nu| \leq
  k$. There exists a ReLU neural network
  $\psi:\R^d \to \R$ with width $9(N+1)+2k-1$,
  depth $7kL(k-1)+1$ and all parameters bounded by $3(k+1)C^k(40N+40)^{2}$
  such that for all $x=(x_1,\ldots,x_d)\in[-C,C]^d$,
  \begin{align*}
    |\psi(x_1,\ldots,x_d) - x^{\nu}| \leq 30C^k(k-1)(N+1)^{-7kL}.
  \end{align*}
\end{lemma}
\begin{proof}
  Assume without loss of generality that $x^{\nu} = x_1^{\nu_1}
  \cdots x_m^{\nu_m}$, where $\nu_j>0$ for $j\in[m]$ and
  $\sum_{j=1}^m \nu_m = k_0 \leq k$.
  Since $\ReLU(a) - \ReLU(-a) = a$ for all $a\in\R$, there
  exists a neural network $\psi_1:\R^d \to \R^{k_0}$
  with width $2k$, depth $1$ and all parameters bounded by 1 such that
  \begin{align*}
    \psi_1(x) \coloneqq (x_1 1_{\nu_1}^\T, \;\cdots \;, x_m
    1_{\nu_m}^\T)^\T \in \R^{k_0}.
  \end{align*}
  Therefore, by Lemma~\ref{lemma:NN-multiplication-general}, there
  exists a neural network $\phi$ with width $9(N+1)+2k-1$,
  depth $7kL(k-1)$ and all parameters bounded by $3C^k(40N+40)^{2}$
  such that
  \begin{align*}
    \bigl|\phi\bigl(\psi_1(x)\bigr) - x^{\nu}\bigr|
    \leq 30C^k(k-1)(N+1)^{-7kL}.
  \end{align*}
  Finally, by network composition, we have that $\psi \coloneqq \phi
  \circ \psi_1$ is a neural network with width $9(N+1)+2k-1$,
  depth $7kL(k-1)+1$ and all parameters bounded by $3(k+1)C^k(40N+40)^{2}$.
\end{proof}

\subsection{Transformer Construction}
\label{sec:transformer_construction}

In the next lemma, we construct a transformer
that yields the centred covariates
and a scaled version of the kernel matrix.

\begin{lemma} \label{lemma:TF-difference-sqrt-kernel}
  Let $n, d \in \N$ and take
  $X_1, \ldots, X_n \in [0, 1]^d$ and $Y_1, \ldots, Y_n \in \R$.
  Write $\bX \coloneqq (X_1, \ldots, X_n)^\T \in \R^{n \times d}$
  and $Y \coloneqq (Y_1, \ldots, Y_n)^\T \in \R^n$.
  Let $h > 0$ and define $K_h(\bX, X_{n+1}) \in \R^n$ by
  $K_h(\bX, X_{n+1})_i \coloneqq K\bigl((X_i - X_{n+1})/h\bigr) / h^d$
  for $i \in [n]$, where $K: \R^d \to \R$ is given by
  $K(x) \coloneqq (1 - \|x\|_1)_+^2$.
  Let $\sqrt{K_h(\bX, X_{n+1})} \in \R^n$ be defined entrywise.
  Take $\de, \dffn \in \N$ with $\de \geq 2d + 4$ and
  $\dffn \geq 2d + 2$, and let
  $B \coloneqq 1 \lor (d/h) \lor (n^{-1/2}h^{-d/2})$.
  Then there exists a transformer
  $\TF \in \cT(\de, \dffn, 3, B)$ such that
  \begin{align*}
    &\TF \circ \Embed_\de \bigl((X_i, Y_i)_{i \in [n]}, X_{n+1}\bigr) \\
    &\hspace{0.2cm}=
    \begin{pmatrix}
      \bX & Y &
      \bigl(\bX - 1_n X_{n+1}^\T\bigr)/h
      & n^{-1/2}\sqrt{K_h(\bX, X_{n+1})} & \bm0_{n\times(\de - 2d - 4)}
      & 1_n & 0_n \\
      X_{n+1}^\T & 0 & 0_d^\T & n^{-1/2}h^{-d/2}
      & 0_{\de - 2d - 4}^\T & 1 & 1
    \end{pmatrix}
    \in \R^{(n+1) \times \de}.
  \end{align*}
\end{lemma}
\begin{proof}
  The input to the desired transformer can be written as
  \begin{align*}
    \bZ_{\rm in}
    \coloneqq
    \Embed_\de \bigl((X_i, Y_i)_{i \in [n]}, X_{n+1}\bigr)
    =
    \begin{pmatrix}
      \bX & Y & \bm0_{n\times(\de - d - 2)} & 0_n \\
      X_{n+1}^\T & 0 & 0_{\de - d - 2}^\T & 1
    \end{pmatrix}
    \in \R^{(n+1) \times \de}.
  \end{align*}
  Let $\Attn^{(1)}$ be a linear attention layer as in
  Definition~\ref{def:attention} with query, key and value
  matrices all zero, so that $\Attn^{(1)}$ is the
  identity map on $\R^{(n+1) \times \de}$.
  Define $b_1^{(1)} \coloneqq -1_{\dffn} \in \R^{\dffn}$
  and
  $b_2^{(1)} \coloneqq (0_{\de - 2}^\T, 1, 0)^\T \in \R^{\de}$,
  and let
  %
  %
  \begin{align*}
    \bW_1^{(1)}
    &\coloneqq
    \begin{pmatrix}
      \bI_{d \times d} & \bm0_{d \times (\de - d - 1)} & 1_d \\
      \bm0_{(\dffn - d)\times d} & \bm0_{(\dffn - d)\times(\de - d
      - 1)} & 0_{\dffn - d}
    \end{pmatrix}
    \in \R^{\dffn \times \de}, \\
    \bW_2^{(1)}
    &\coloneqq
    \begin{pmatrix}
      \bm0_{(d + 1)\times d} & \bm0_{(d + 1)\times (\dffn - d)} \\
      \bI_{d \times d} & \bm0_{d\times (\dffn - d)} \\
      \bm0_{(\de - 2d - 1)\times d} & \bm0_{(\de - 2d - 1)\times (\dffn - d)}
    \end{pmatrix}
    \in \R^{\de \times \dffn}.
  \end{align*}
  %
  %
  Then, since $X_i \in [0, 1]^d$ for $i \in [n+1]$, we have
  \begin{align*}
    \ReLU \bigl(
      \bW_1^{(1)} {\bZ}_{\rm in}^\T
      + b_1^{(1)} 1_{n+1}^\T
    \bigr)
    &= \ReLU
    \begin{pmatrix}
      \bX^\T - 1_d 1_n^\T & X_{n+1} \\
      - 1_{\dffn - d} 1_n^\T & - 1_{\dffn - d} \\
    \end{pmatrix}
    \iftoggle{journal}{}{\\&}
    =
    \begin{pmatrix}
      \bm0_{d\times n} & X_{n+1} \\
      \bm0_{(\dffn - d)\times n} & 0_{\dffn - d}
    \end{pmatrix}
    \in \R^{\dffn \times (n + 1)}.
  \end{align*}
  Hence, writing $\FFN^{(1)} \coloneqq
  \FFN_{\bW_1^{(1)}, \bW_2^{(1)}, b_1^{(1)}, b_2^{(1)}}$,
  the first transformer block gives
  \begin{align*}
    \bZ^{(1)}
    &\coloneqq \FFN^{(1)} \circ \Attn^{(1)} (\bZ_{\mathrm{in}})
    = \FFN^{(1)} (\bZ_{\mathrm{in}}) \\
    &\,=
    {\bZ}_{\rm in}
    + \bigl\{
      \bW_2^{(1)}\ReLU\bigl(
        \bW_1^{(1)} {\bZ}_{\rm in}^\T
        + b_1^{(1)} 1_{n+1}^\T
      \bigr)
    + b_2^{(1)} 1_{n+1}^{\T}\bigr\}^\T \\
    &\,=
    \begin{pmatrix}
      \bX & Y & \bm0_{n\times d} & \bm0_{n \times(\de - 2d - 3)} & 1_n & 0_n \\
      X_{n+1}^\T & 0 & X_{n+1}^\T & 0_{\de - 2d - 3}^\T & 1 & 1
    \end{pmatrix}
    \in \R^{(n+1) \times \de}.
  \end{align*}
  %
  %
  For the second attention layer, let
  $\Attn^{(2)} \coloneqq \Attn_{\bQ^{(2)}, \bK^{(2)}, \bV^{(2)}}$
  where
  \begin{gather*}
    \bQ^{(2)}
    \coloneqq
    \frac{1}{\sqrt{\de}}
    \begin{pmatrix}
      \bm0_{(\de-2)\times\de} \\
      1_{\de}^\T \\
      -1_{\de}^\T
    \end{pmatrix}
    \in \R^{\de \times \de},
    \qquad\qquad
    \bK^{(2)}
    \coloneqq
    \frac{1}{\sqrt{\de}}
    \begin{pmatrix}
      \bm0_{(\de-2)\times\de} \\
      1_{\de}^\T \\
      0_{\de}^\T
    \end{pmatrix}
    \in \R^{\de \times \de}, \\
    \bV^{(2)}
    \coloneqq
    \begin{pmatrix}
      \bm0_{(d+1)\times(d+1)} & \bm0_{(d+1)\times d} & \bm
      0_{(d+1)\times(\de - 2d - 1)} \\
      \bm0_{d\times(d+1)} & \bI_{d \times d} & \bm0_{d\times(\de - 2d - 1)} \\
      \bm0_{(\de - 2d - 1)\times(d+1)} & \bm0_{(\de - 2d - 1) \times d}
      & \bm0_{(\de - 2d - 1)\times(\de - 2d - 1)}
    \end{pmatrix}
    \in \R^{\de \times \de}.
  \end{gather*}
  Then we may write
  \begin{gather*}
    \bZ^{(1)} \bQ^{(2)}
    =
    \frac{1}{\sqrt{\de}}
    \begin{pmatrix}
      1_n 1_{\de}^\T \\
      0_{\de}^\T
    \end{pmatrix},\quad
    \bZ^{(1)} \bK^{(2)}
    =
    \frac{1}{\sqrt{\de}}
    1_{n+1} 1_{\de}^\T, \\
    \bZ^{(1)} \bV^{(2)}
    =
    \begin{pmatrix}
      \bm0_{n\times(d+1)} & \bm0_{n\times d} & \bm0_{n\times(\de - 2d - 1)} \\
      0_{d+1}^\T & X_{n+1}^\T & 0_{\de - 2d - 1}^\T
    \end{pmatrix}.
  \end{gather*}
  Therefore, the output of the second attention layer is
  \begin{align*}
    \Attn^{(2)}(\bZ^{(1)})
    &=
    \bZ^{(1)}
    + \bZ^{(1)} \bQ^{(2)} (\bZ^{(1)} \bK^{(2)})^\T
    \bZ^{(1)} \bV^{(2)} \\
    &=
    \begin{pmatrix}
      \bX & Y & 1_n X_{n+1}^\T & \bm0_{n \times(\de - 2d - 3)} & 1_n & 0_n \\
      X_{n+1}^\T & 0 & X_{n+1}^\T & 0_{\de - 2d - 3}^\T & 1 & 1
    \end{pmatrix}
    \in \R^{(n+1) \times \de}.
  \end{align*}
  For the second feed-forward layer, let
  $b_1^{(2)} \coloneqq 0_{\dffn} \in \R^{\dffn}$
  and define
  \begin{align*}
    \bW_1^{(2)}
    &\coloneqq
    \begin{pmatrix}
      \bI_{d \times d}/h & 0_d & - \bI_{d \times d}/h
      & \bm0_{d\times(\de - 2d - 1)} \\
      -\bI_{d \times d}/h & 0_d & \bI_{d \times d}/h
      & \bm0_{d\times(\de - 2d - 1)} \\
      \bm0_{(\dffn - 2d)\times d} & 0_{\dffn - 2d}
      & \bm0_{(\dffn - 2d)\times d}
      & \bm0_{(\dffn - 2d)\times (\de - 2d - 1)}
    \end{pmatrix}
    \in \R^{\dffn \times \de}.
  \end{align*}
  Then
  \begin{align*}
    &
    \ReLU\Bigl(\bW_1^{(2)}
      \Attn^{(2)}\bigl(\bZ^{(1)}\bigr)^\T
      + b_1^{(2)} 1_{n+1}^\T
    \Bigr) =
    \begin{pmatrix}
      \ReLU(\bX^\T - X_{n+1} 1_n^\T) / h & 0_d \\
      \ReLU(X_{n+1} 1_n^\T - \bX^\T) / h & 0_d \\
      \bm0_{(\dffn - 2d)\times n} & 0_{\dffn - 2d}
    \end{pmatrix}.
  \end{align*}
  %
  %
  Next, let
  $b_2^{(2)} \coloneqq (0_{2d+1}^\T, 1, 0_{\de-2d-2}^\T)^\T \in \R^{\de}$
  and define
  \begin{align*}
    \bW_2^{(2)}
    &\coloneqq
    \begin{pmatrix}
      \bm0_{(2d+1)\times d} & \bm0_{(2d+1)\times d} &
      \bm0_{(2d+1)\times(\dffn - 2d)} \\
      -1_d^\T & -1_d^\T & 0_{\dffn - 2d}^\T \\
      \bm0_{(\de-2d-2)\times d} & \bm0_{(\de-2d-2)\times d} &
      \bm0_{(\de-2d-2)\times(\dffn-2d)}
    \end{pmatrix}
    \in \R^{\de \times \dffn}.
  \end{align*}
  Defining $\tilde K \in \R^n$ by
  $\tilde K^\T \coloneqq
  1_n^\T - 1_d^\T \ReLU(\bX^\T - X_{n+1} 1_n^\T) / h
  - 1_d^\T \ReLU(X_{n+1} 1_n^\T - \bX^\T) / h$,
  it follows that
  \begin{align*}
    &\bW_2^{(2)}
    \ReLU\Bigl(
      \bW_1^{(2)}
      \Attn^{(2)}\bigl(\bZ^{(1)}\bigr)^\T
      + b_1^{(2)} 1_{n+1}^\T
    \Bigr)
    + b_2^{(2)} 1_{n+1}^\T
    =
    \begin{pmatrix}
      \bm0_{(2d+1)\times n} & 0_{2d+1} \\
      \tilde K^\T & 1 \\
      \bm0_{(\de-2d-2)\times n} & 0_{\de-2d-2}
    \end{pmatrix}.
  \end{align*}
  Writing $\FFN^{(2)} \coloneqq
  \FFN_{\bW_1^{(2)}, \bW_2^{(2)}, b_1^{(2)}, b_2^{(2)}}$,
  the output of the second transformer block is
  \begin{align*}
    \bZ^{(2)}
    &\coloneqq
    \FFN^{(2)} \circ \Attn^{(2)}\bigl(\bZ^{(1)}\bigr) \\
    &\,=
    \Attn^{(2)}\bigl(\bZ^{(1)}\bigr)
    + \Bigl\{
      \bW_2^{(2)}
      \ReLU\Bigl(
        \bW_1^{(2)}
        \Attn^{(2)}\bigl(\bZ^{(1)}\bigr)^\T
        + b_1^{(2)} 1_{n+1}^\T
      \Bigr)
      + b_2^{(2)} 1_{n+1}^\T
    \Bigr\}^\T \\
    &\,=
    \begin{pmatrix}
      \bX & Y & 1_n X_{n+1}^\T & \tilde K & \bm0_{n\times(\de - 2d - 4)}
      & 1_n & 0_n \\
      X_{n+1}^\T & 0 & X_{n+1}^\T & 1 & 0_{\de - 2d - 4}^\T & 1 & 1
    \end{pmatrix}
    \in \R^{(n+1) \times \de}.
  \end{align*}
  We take the third attention layer
  $\Attn^{(3)}$ to be the identity function.
  For the third feed-forward layer, let
  $b_1^{(3)} \coloneqq
  (0, d/h, 1_d^\T/h, 0_{\dffn-d-2}^\T)^\T \in \R^{\dffn}$
  and
  \begin{align*}
    \bW_1^{(3)}
    \!
    \coloneqq
    \begin{pmatrix}
      \bm0_{2\times d} & 0_2 & \bm0_{2\times d} & 1_2 &
      \bm0_{2\times(\de-2d-2)} \\
      \bI_{d \times d}/h & 0_d & -\bI_{d \times d}/h & 0_d &
      \bm0_{d\times(\de-2d-2)} \\
      \bm0_{d\times d} & 0_d & \bI_{d \times d} & 0_d &
      \bm0_{d\times(\de-2d-2)} \\
      \bm0_{(\dffn-2d-2)\times d} & 0_{\dffn-2d-2} &
      \bm0_{(\dffn-2d-2)\times d} &
      0_{\dffn-2d-2} & \bm0_{(\dffn-2d-2)\times(\de-2d-2)}
    \end{pmatrix}.
  \end{align*}
  Then
  \begin{equation*}
    \bW_1^{(3)}
    \bZ^{(2),\T}
    + b_1^{(3)} 1_{n+1}^\T
    =
    \begin{pmatrix}
      \tilde K^\T & 1 \\
      \tilde K^\T + d 1_n^\T / h & 1 + d/h \\
      \bigl(\bX^\T - X_{n+1} 1_n^\T + 1_d 1_n^\T\bigr)/h & 1_d / h \\
      X_{n+1} 1_n^\T & X_{n+1} \\
      \bm0_{(\dffn-2d-2)\times n} & 0_{\dffn-2d-2}
    \end{pmatrix}
    \in \R^{\dffn \times (n+1)}.
  \end{equation*}
  Moreover,
  writing $\tilde K_i$ for the $i$th
  component of $\tilde K$ for $i \in [n]$, we have
  \begin{align*}
    \ReLU(\tilde K_i)
    &=
    \biggl\{
      1
      - \sum_{j=1}^d \biggl(\frac{X_{i, j} - X_{n+1, j}}{h}\biggr)_+
      - \sum_{j=1}^d \biggl(\frac{X_{n+1, j} - X_{i, j}}{h}\biggr)_+
    \biggr\}_+ \\
    &=
    \biggl\{
      1 - \sum_{j=1}^d \biggl|\frac{X_{i, j} - X_{n+1, j}}{h}\biggr|
    \biggr\}_+
    =
    \biggl\{
      1 - \biggl\|\frac{X_i - X_{n+1}}{h}\biggr\|_1
    \biggr\}_+ \\
    &=
    \sqrt{h^d K_h(\bX, X_{n+1})_i}.
  \end{align*}
  Therefore, interpreting the square root as acting entrywise,
  $\ReLU(\tilde K) = \sqrt{h^d K_h(\bX, X_{n+1})}$.
  As $X_i \in [0, 1]^d$ for each $i \in [n+1]$ and
  $\tilde K_i \geq -d/h$ for $i \in [n]$, we have
  \begin{align*}
    \ReLU\Bigl(
      \bW_1^{(3)} \bZ^{(2),\T} + b_1^{(3)} 1_{n+1}^\T
    \Bigr)
    &=
    \begin{pmatrix}
      \sqrt{h^d K_h(\bX, X_{n+1})^\T} & 1 \\
      \tilde K^\T + d 1_n^\T / h & 1 + d/h \\
      \bigl(1_d 1_n^\T + \bX^\T - X_{n+1} 1_n^\T\bigr)/h & 1_d / h \\
      X_{n+1} 1_n^\T & X_{n+1} \\
      \bm0_{(\dffn-2d-2)\times n} & 0_{\dffn-2d-2}
    \end{pmatrix}
    \in \R^{\dffn \times (n+1)}.
  \end{align*}
  Now let $b_2^{(3)} \coloneqq (0_{d+1}^\T, -1_d^\T/h, d/h, 0_{\de-2d-2}^\T)^\T
  \in \R^{\de}$ and
  \begin{align*}
    \bW_2^{(3)}
    \!
    \coloneqq
    \!\!
    \begin{pmatrix}
      0_{d+1} & 0_{d+1} & \bm0_{(d+1)\times d} & \bm0_{(d+1)\times d}
      & \bm0_{(d+1)\times (\dffn-2d-2)} \\
      0_{d} & 0_{d} & \bI_{d \times d} & -\bI_{d \times d}
      & \bm0_{d\times (\dffn-2d-2)} \\
      n^{-1/2}h^{-d/2} & -1 & 0_d^\T & 0_d^\T & 0_{\dffn-2d-2}^\T \\
      0_{\de-2d-2} & 0_{\de-2d-2} & \bm0_{(\de-2d-2)\times d}
      & \bm0_{(\de-2d-2)\times d} & \bm0_{(\de-2d-2)\times(\dffn-2d-2)} \\
    \end{pmatrix}
    \!\!
    \in
    \R^{\de \times \dffn}.
  \end{align*}
  Then
  \begin{align*}
    \bW_2^{(3)}
    \ReLU\Bigl(
      \bW_1^{(3)} \bZ^{(2),\T} &+ b_1^{(3)} 1_{n+1}^\T
    \Bigr)
    + b_2^{(3)} 1_{n+1}^\T \\
    &\quad=
    \begin{pmatrix}
      \bm0_{(d+1)\times n} & 0_{d+1} \\
      \bigl(\bX^\T - X_{n+1} 1_n^\T\bigr)/h - X_{n+1} 1_n^\T & - X_{n+1} \\
      n^{-1/2}\sqrt{K_h(\bX, X_{n+1})^\T} - \tilde K^\T
      & n^{-1/2}h^{-d/2} - 1 \\
      \bm0_{(\de-2d-2)\times n} & 0_{\de-2d-2}
    \end{pmatrix}
    \in \R^{\de \times (n+1)}.
  \end{align*}
  Finally, with $\FFN^{(3)} \coloneqq
  \FFN_{\bW_1^{(3)}, \bW_2^{(3)}, b_1^{(3)}, b_2^{(3)}}$,
  the output of the third transformer block is
  \begin{align*}
    \bZ^{(3)}
    &\coloneqq
    \FFN^{(3)} \circ \Attn^{(3)}\bigl(\bZ^{(2)}\bigr)
    = \FFN^{(3)} \bigl(\bZ^{(2)}\bigr) \\
    &\,=
    \bZ^{(2)}
    + \Bigl\{
      \bW_2^{(3)}
      \ReLU\Bigl(
        \bW_1^{(3)}
        \bZ^{(2),\T}
        + b_1^{(3)} 1_{n+1}^\T
      \Bigr)
      + b_2^{(3)} 1_{n+1}^\T
    \Bigr\}^\T \\
    &\,=
    \begin{pmatrix}
      \bX & Y &
      \bigl(\bX - 1_n X_{n+1}^\T\bigr)/h
      & n^{-1/2}\sqrt{K_h(\bX, X_{n+1})} & \bm0_{n\times(\de - 2d - 4)}
      & 1_n & 0_n \\
      X_{n+1}^\T & 0 & 0_d^\T & n^{-1/2} h^{-d/2}
      & 0_{\de - 2d - 4}^\T & 1 & 1
    \end{pmatrix},
  \end{align*}
  as required.
\end{proof}

The following lemma constructs the kernel-weighted polynomials
and responses
using a transformer.
Let $h\coloneqq n^{-1/(2\alpha+d)}$, let
$\bP_h(X_{n+1}) \in \R^{n\times D}$ be defined as in
Section~\ref{sec:locpol_setup}, and write
$\tilde{\bX} \coloneqq n^{-1/2}\sqrt{\bK_h(X_{n+1})} \bP_h(X_{n+1})
\in \R^{n\times D}$ and
$\tilde{Y} \coloneqq n^{-1/2}\sqrt{\bK_h(X_{n+1})} Y \in \R^n$.

\begin{lemma} \label{lemma:approx-weighted-poly}
  Let $L_0 > 0$, $p\coloneqq \lceil\alpha\rceil$,
  $D\coloneqq \binom{d+p}{p}$,
  $\de \coloneqq 2d+2D+5$,
  $\dffn \coloneqq 3(D+1)(28+2p)$,
  $L \coloneqq 7p(p+1)\lceil L_0\log n \rceil + 4$
  and $B\coloneqq 3d\cdot 120^2 \cdot (p+2)(2M+2)^{p+1} n^{(p+1)/(2\alpha+d)}$.
  Then there exists a transformer
  $\TF \in \cT(\de,\dffn,L,B)$ such that
  \begin{align*}
    \TF &\circ \Embed_\de \bigl((X_i, Y_i)_{i \in [n]}, X_{n+1}\bigr) \\
    &=
    \begin{pmatrix}
      \bX & Y &
      \bigl(\bX - 1_n X_{n+1}^\T\bigr)/h
      & n^{-1/2}\sqrt{K_h(\bX, X_{n+1})} & \check{\bX}
      & \check{Y} & \bm0_{n\times D}
      & 1_n & 0_n \\
      X_{n+1}^\T & 0 & 0_d^\T & n^{-1/2}h^{-d/2} & a^\T & b & 0_D^\T & 1 & 1
    \end{pmatrix},
  \end{align*}
  where $\check{\bX} \in \R^{n\times D}$, $\check{Y} \in \R^n$,
  $a\in \R^D$ and $b\in\R$ satisfy
  \begin{align*}
    \|\check{\bX} - \tilde{\bX}\|_{\max}
    \vee
    \|\check{Y} - \tilde{Y}\|_{\infty}
    \vee \|a\|_{\infty} \vee |b|
    \leq 30p\cdot (2M+2)^{p+1} \cdot n^{-7(p+1)L_0 + 2}.
  \end{align*}
\end{lemma}

\begin{proof}
  By Lemma~\ref{lemma:TF-difference-sqrt-kernel},
  there exists $\TF^{(1)} \in \cT(\de,\dffn,3,B)$ such that
  \begin{align*}
    &\bZ^{(1)}
    \coloneqq
    \TF^{(1)} \circ \Embed_\de \bigl((X_i, Y_i)_{i \in [n]}, X_{n+1}\bigr)\\
    &=
    \begin{pmatrix}
      \bX & Y &
      \bigl(\bX - 1_n X_{n+1}^\T\bigr)/h
      & n^{-1/2}\sqrt{K_h(\bX, X_{n+1})} & \bm0_{n\times(2D+1)}
      & 1_n & 0_n \\
      X_{n+1}^\T & 0 & 0_d^\T & n^{-1/2}h^{-d/2} & 0_{2D+1}^\T & 1 & 1
    \end{pmatrix}
    \in \R^{(n+1) \times \de}.
  \end{align*}
  Next, note that each entry of the $i$th row of $\tilde{\bX}$,
  as well as the $i$th entry of $\tilde Y$, is a monomial in the $i$th row of
  $\bigl(Y,\, (\bX - 1_n X_{n+1}^\T)/h,\,
  n^{-1/2}\sqrt{K_h(\bX, X_{n+1})}\bigr)$
  of degree at most $p+1$.
  Thus, by Lemma~\ref{lemma:ignore_residual} and by applying
  Lemma~\ref{lemma:NN_polynomial}
  with $N = 2$,
  $k = p + 1$,
  $L = \lceil L_0 \log n \rceil$
  and $C = 2(M+1)n^{1/(2\alpha+d)}$ therein,
  there exists $\TF^{(2)} \in \cT(\de,\dffn,L-3,B)$
  with all attention layers being the identity map such that
  \begin{align*}
    \TF^{(2)} (\bZ^{(1)}) =
    \begin{pmatrix}
      \bX & Y &
      \bigl(\bX - 1_n X_{n+1}^\T\bigr)/h
      & n^{-1/2}\sqrt{K_h(\bX, X_{n+1})} & \check{\bX} & \check{Y}
      & \bm0_{n\times D}
      & 1_n & 0_n \\
      X_{n+1}^\T & 0 & 0_d^\T & n^{-1/2}h^{-d/2} & a^\T & b & 0_D^\T & 1 & 1
    \end{pmatrix},
  \end{align*}
  where
  \begin{align*}
    \|\check{\bX} - \tilde{\bX}\|_{\max}
    \vee \|\check{Y} - \tilde{Y}\|_{\infty}
    \vee \|a\|_{\infty}
    \vee |b|
    &\leq
    30\bigl(2(M+1)n^{1/(2\alpha+d)}\bigr)^{p+1} p\cdot 3^{-7(p+1)L_0\log n}\\
    &\leq 30p\cdot (2M+2)^{p+1} \cdot n^{-7(p+1)L_0 + 2}.
  \end{align*}
  Thus $\TF \coloneqq \TF^{(2)} \circ \TF^{(1)} \in \cT(\de,\dffn,L,B)$
  satisfies the conditions of the lemma.
\end{proof}

The next lemma quantifies the rate of convergence of
the gradient descent algorithm for
strongly convex functions with errors in the gradients.
A similar result can be found in \citet{bai2023transformers}.

\begin{lemma}\label{lemma:GD-error}
  Let $C_1,C_2>0$, $\epsilon \geq 0$, $m \in \N$,
  and let $f:\R^m \to \R$. Suppose that $w\mapsto f(w) - (C_1/2)\|w\|_2^2$ is
  convex, let $w_*\coloneqq \argmin_{w \in \R^m} f(w)$,
  and let $R \geq 2\|w_*\|_2$.
  Further suppose that $f$ is differentiable with $\|\nabla f(w) - \nabla
  f(w')\|_2 \leq C_2\|w-w'\|_2$ for all $w,w'\in\R^m$ such that
  $\|w\|_2,\|w'\|_2 \leq R$. Let $w_0 \coloneqq 0_m$, and for
  $t\in\N_0$, define $w_{t+1} \coloneqq w_t - g_t / C_2$
  where $g_t \in \R^m$ satisfies
  $\|g_t - \nabla f(w_t)\|_2 \leq \epsilon$.
  Then, for any $T\leq RC_2/(2\epsilon)$, we have
  \begin{equation*}
    \|w_T\|_2 \leq R \quad\text{and}\quad
    \|w_T - w_*\|_2
    \leq \exp\biggl(-\frac{C_1 T}{2C_2}\biggr) \|w_*\|_2
    + \frac{T\epsilon}{C_2}.
  \end{equation*}
\end{lemma}

\begin{proof}
  Let $v_0 \coloneqq 0_m$ and $v_{t+1} \coloneqq v_t - \nabla f(v_t)/C_2$
  for $t\in\N_0$.
  Then, by \citet[Theorem~3.10]{bubeck2015convex}, we have
  \begin{equation*}
    \|v_T - w_*\|_2 \leq \exp\biggl(-\frac{C_1 T}{2C_2}\biggr) \|w_*\|_2.
  \end{equation*}
  Moreover, by \citet[Lemma~D.1]{bai2023transformers},
  $\|w_T\|_2 \leq R$ and $\|w_T - v_T\|_2 \leq T\epsilon/C_2$.
  Therefore,
  \begin{align*}
    \|w_T - w_*\|_2 \leq \|v_T - w_*\|_2 + \|v_T - w_T\|_2
    \leq \exp\biggl(-\frac{C_1 T}{2C_2}\biggr) \|w_*\|_2
    + \frac{T\epsilon}{C_2},
  \end{align*}
  as required.
\end{proof}

Recall that we wish to solve the weighted least squares
problem~\eqref{eq:weighted-least-square}.
For $w \in \R^D$, let
\begin{equation}
  \label{eq:least_squares}
  f(w) \coloneqq \bigl\| \tilde{Y} - \tilde{\bX} w \bigr\|_2^2,
\end{equation}
and let $w_* \coloneqq \argmin_{w\in\R^D} f(w)$.
Then the gradient of $f$ at $w$ is
\begin{equation}
  \label{eq:least_squares_grad}
  \nabla f(w)
  = 2\bigl(
    \tilde{\bX}^\T \tilde{\bX} w - \tilde{\bX}^\T \tilde{Y}
  \bigr) \in \R^D,
\end{equation}
and we can write
\begin{equation}
  \label{eq:gradient_calculation}
  1_{n+1} \nabla f(w)^\T = 2
  \begin{pmatrix}
    w^\T & - \frac{1_D^\T}{D}\\
    \vdots & \vdots\\
    w^\T & -\frac{1_D^\T}{D}
  \end{pmatrix}
  \begin{pmatrix}
    \tilde{\bX}^\T & 0_D\\
    1_D \tilde{Y}^\T & 0_D\\
  \end{pmatrix}
  \begin{pmatrix}
    \tilde{\bX} \\ 0_D^\T
  \end{pmatrix}.
\end{equation}
The right-hand side of \eqref{eq:gradient_calculation}
is a product of three matrices, and so
can be implemented by one linear attention layer
with an appropriate input matrix.
However, we can only
approximate $\tilde{\bX}$ and $\tilde{Y}$ by a
transformer, and the approximation errors are controlled by
Lemma~\ref{lemma:approx-weighted-poly}.
We combine these ideas in
the following lemma to show that one attention layer can approximate one step
of gradient descent for minimising $f$,
with an error in the gradient vector.

\begin{lemma} \label{lemma:transformer-GD}
  Let $D,\de,\dffn\in\N$ be defined as in
  Lemma~\ref{lemma:approx-weighted-poly}. Let $\eta>0$,
  $B\coloneqq (2\eta) \vee 1$ and let $w\in\R^D$ be such that $\|w\|_2 \leq R$
  for some $R>0$. Define
  \begin{align*}
    \bZ \coloneqq
    \begin{pmatrix}
      \bX & Y &
      \bigl(\bX - 1_n X_{n+1}^\T\bigr)/h
      & n^{-1/2}\sqrt{K_h(\bX, X_{n+1})} & \check{\bX} & \check{Y} & 1_n w^\T
      & 1_n & 0_n \\
      X_{n+1}^\T & 0 & 0_d^\T & n^{-1/2}h^{-d/2} & a^\T & b & w^\T & 1 & 1
    \end{pmatrix} \in\R^{(n+1)\times \de},
  \end{align*}
  where $\check{\bX} \in \R^{n\times D}$, $\check{Y} \in \R^n$,
  $a\in \R^D$ and $b\in\R$ satisfy
  \begin{align}
    \label{eq:tilde-check-error}
    \|\check{\bX} - \tilde{\bX}\|_{\max}
    \vee \|\check{Y} - \tilde{Y}\|_{\infty}
    \vee \|a\|_{\infty} \vee |b| \leq \xi,
  \end{align}
  for some $\xi\in[0,1]$.
  Then there exists $\TF \in \cT(\de,\dffn,1,B)$ such that
  \begin{align*}
    \begingroup
    \TF(\bZ) =
    \begin{pmatrix}
      \bX & Y &
      \bigl(\bX - 1_n X_{n+1}^\T\bigr)/h
      & n^{-1/2}\sqrt{K_h(\bX, X_{n+1})} & \check{\bX} & \check{Y}
      & 1_n (w-\eta g)^\T
      & 1_n & 0_n \\
      X_{n+1}^\T & 0 & 0_d^\T & n^{-1/2}h^{-d/2}
      & a^\T & b & (w-\eta g)^\T & 1 & 1
    \end{pmatrix},
    \endgroup
  \end{align*}
  where $g \in \R^D$ satisfies
  \begin{align*}
    \|g - \nabla f(w)\|_2
    \leq 18D(M+3)(R+1)n\xi.
  \end{align*}
\end{lemma}

\begin{proof}
  We can choose
  $\bQ, \bK, \bV \in [-(2\eta\vee 1),(2\eta\vee 1)]^{\de \times \de}$
  such that
  \begin{align*}
    \bZ \bQ &= -2\eta
    \begin{pmatrix}
      w^\T & -\frac{1_D^\T}{D} & 0_{\de-2D}^\T\\
      \vdots & \vdots & \vdots\\
      w^\T & -\frac{1_D^\T}{D} & 0_{\de-2D}^\T
    \end{pmatrix} \in \R^{(n+1) \times \de},\\
    \bZ \bK &=
    \begin{pmatrix}
      \check{\bX} & \check{Y}1_D^\T & \bm0_{n\times(\de-2D)}\\
      a^\T & b1_D^\T & 0_{\de-2D}^\T
    \end{pmatrix} \in \R^{(n+1) \times \de},\\
    \bZ \bV &=
    \begin{pmatrix}
      \bm0_{n\times(\de-D-2)} & \check{\bX} & \bm0_{n\times2}\\
      0_{\de-D-2}^\T & a^\T & 0_2^\T
    \end{pmatrix} \in \R^{(n+1) \times \de}.
  \end{align*}
  Then, writing $g\coloneq 2\bigl(\check{\bX}^\T \check{\bX} w -
  \check{\bX}^\T \check{Y} + (w^\T a - b) a\bigr)$, we have
  \begin{align*}
    \Attn_{\bQ,\bK,\bV}(\bZ)
    &= \bZ + \bZ \bQ
    (\bZ \bK)^\T \bZ \bV
    = \bZ - \eta
    \begin{pmatrix}
      0_{\de-D-2}^\T
      & g^\T  & 0_2^\T\\
      \vdots & \vdots & \vdots \\
      0_{\de-D-2}^\T
      & g^\T  & 0_2^\T
    \end{pmatrix}.
  \end{align*}
  Moreover,
  \begin{align*}
    \bigl\|
    g
    &- \nabla f(w)
    \bigr\|_2
    =
    2
    \bigl\|
    \check{\bX}^\T \check{\bX} w - \check{\bX}^\T \check{Y}
    + (w^\T a - b) a - \tilde{\bX}^\T \tilde{\bX} w
    + \tilde{\bX}^\T \tilde{Y}
    \bigr\|_2\\
    &\leq 2
    \Bigl\{
      \|\check{\bX} - \tilde{\bX}\|_{\op}
      \bigl(\|\check{\bX}\|_{\op}
      + \|\tilde{\bX}\|_{\op}\bigr)
      \|w\|_2
      + \|\check{\bX} - \tilde{\bX}\|_{\op}\|\check{Y}\|_2
      \iftoggle{journal}{}{\\&\qquad\quad}
      +
      \|\tilde{\bX}\|_{\op} \|\check{Y} - \tilde{Y}\|_2
      + \bigl(\|w\|_2\|a\|_2 + |b|\bigr)
    \|a\|_2\Bigr\}\\
    &\leq
    2
    \Bigl\{
      \xi (n+1)D (2 + \xi) \|w\|_2
      + \xi(n+1)(M+\xi+2)\sqrt{D}
      + \bigl(\|w\|_2\xi \sqrt{D} + \xi\bigr)
      \xi\sqrt{D}
    \Bigr\}\\
    &\leq 18D(M+3)(R+1)n\xi,
  \end{align*}
  where the second inequality
  uses~\eqref{eq:tilde-check-error} and the fact that
  $\|\tilde{\bX}\|_{\max} \leq 1$
  and $\|\tilde{Y}\|_{\infty}
  \leq M+1$. Finally, we take the FFN layer to be the
  identity map by setting all its parameters to zero,
  and this proves the claim.
\end{proof}

\begin{proposition} \label{prop:TF-approx-lp}
  Let $p, D,\de,\dffn\in\N$ and $B>0$ be defined as in
  Lemma~\ref{lemma:approx-weighted-poly}, and let $\tilde m_n$ denote the
  $M$-truncated local polynomial estimator
  defined in Section~\ref{sec:locpol_setup} with degree $p$,
  kernel $K(x)\coloneqq (1-\|x\|_1)_+^2$ and bandwidth
  $h \coloneqq n^{-1/(2\alpha+d)}$.
  There exists $C>0$, depending only on $d, \alpha, M, c_X, C_X$,
  such that if
  $L\coloneqq \lceil C \log (en) \rceil$,
  then we can find a transformer
  $\TF \in \cT(\de,\dffn,L,B)$ satisfying
  \begin{align*}
    \bigl|
    \Read_{M,d} \circ \TF \circ
    \Embed_\de \bigl((X_i, Y_i)_{i \in [n]}, X_{n+1}\bigr)
    - \tilde{m}_n(X_{n+1})
    \bigr| \leq \frac{C}{n}
  \end{align*}
  with probability at least
  $1-n^{C/(2\alpha+d)} \exp\bigl(-n^{2\alpha/(2\alpha+d)}/C\bigr)$.
\end{proposition}

\begin{proof}
  In this proof, $n_0,n_1,C_0, C_1,\ldots$ are positive quantities depending
  only on $d, \alpha, M, c_X, C_X$.
  By Lemma \ref{lem:locpol_Hhat}, there exists $C_0\geq 1$ such that
  if $n\geq C_0^{2\alpha+d}$, then the event
  \begin{align*}
    \mathcal{E}_0
    \coloneqq
    \Bigl\{
      C_0^{-1} \leq
      \lambda_{\min}\bigl(\tilde{\bX}^\T \tilde{\bX}\bigr)
      \leq
      \lambda_{\max}\bigl(\tilde{\bX}^\T \tilde{\bX}\bigr)
      \leq C_0
    \Bigr\}
  \end{align*}
  has probability at least
  $1-n^{C_0/(2\alpha+d)} \exp\bigl(-n^{2\alpha/(2\alpha+d)}/C_0\bigr)$.
  Let $w_* \coloneqq \argmin_{w\in\R^D} \|\tilde{Y} - \tilde{\bX} w\|_2^2$
  be the minimiser of~$f$,
  as defined in \eqref{eq:least_squares}.
  Then, on the event $\mathcal{E}_0$, we have
  \begin{align}
    \label{eq:w_*-bound}
    \|w_*\|_2 =
    \bigl\|
    (\tilde{\bX}^\T \tilde{\bX})^{-1}
    \tilde{\bX}^\T \tilde{Y}
    \bigr\|_2
    \leq
    C_0\|\tilde{\bX}^\T \tilde{Y}\|_2
    \leq C_0(M+1)n\sqrt{D},
  \end{align}
  where the final inequality uses the fact that
  $\|\tilde{\bX}\|_{\max} \leq 1$
  and $\|\tilde{Y}\|_{\infty}
  \leq M+1$.
  Moreover, on the event $\mathcal{E}_0$,
  the map
  $w \mapsto f(w) - \|w\|_2^2/C_0$ is convex
  and $\|\nabla f(w) - \nabla f(w')\|_2
  \leq 2C_0 \|w-w'\|_2$ for all
  $w,w'\in\R^D$, by \eqref{eq:least_squares_grad}.
  We define $R\coloneqq 2C_0(M+1)n\sqrt{D}$ and
  work on the
  event $\mathcal{E}_0$ for the rest of the proof.

  By Lemma~\ref{lemma:approx-weighted-poly}, there exists a transformer
  $\TF^{(0)} \in \cT(\de,\dffn,L^{(0)},B)$ where
  $L^{(0)}\coloneqq 7p(p+1)\lceil(\log n)/(p+1)\rceil + 4$ such that
  \begin{align}
    \bZ^{(0)} &\coloneqq \TF^{(0)} \circ \Embed_\de
    \bigl((X_i, Y_i)_{i \in [n]}, X_{n+1}\bigr) \nonumber \\
    &\,=
    \begin{pmatrix}
      \bX & Y &
      \bigl(\bX - 1_n X_{n+1}^\T\bigr)/h
      & n^{-1/2}\sqrt{K_h(\bX, X_{n+1})} & \check{\bX}
      & \check{Y} & 1_n w_0^\T
      & 1_n & 0_n \\
      X_{n+1}^\T & 0 & 0_d^\T & n^{-1/2}h^{-d/2} & a^\T & b & w_0^\T & 1 & 1
    \end{pmatrix}, \label{eq:GD-step-0}
  \end{align}
  where $w_0 \coloneqq 0_D$, and where $\check{\bX} \in \R^{n\times D}$,
  $\check{Y} \in \R^n$,
  $a\in \R^D$ and $b\in\R$ satisfy
  \begin{align}
    \|\check{\bX} - \tilde{\bX}\|_{\max}
    \vee
    \|\check{Y} - \tilde{Y}\|_{\infty}
    \vee \|a\|_{\infty} \vee |b|
    \leq 30p\cdot (2M+2)^{p+1} \cdot n^{-5} \eqqcolon \xi. \label{eq:xi-def}
  \end{align}
  Next, define $\epsilon\coloneqq 18D(M+3)(R+1)n\xi$.
  Assume initially that $n \geq n_0$, where $n_0$
  is large enough that $\xi \leq 1$.
  We claim that for every
  $T \in \bigl\{0,1,\ldots,\lfloor RC_0/\epsilon\rfloor\bigr\}$,
  there exist $\TF^{(0)} \in \cT(\de,\dffn,L^{(0)},B)$,
  $\TF^{(1)},\ldots,\TF^{(T)} \in \cT(\de,\dffn,1,B)$,
  $w_0,\ldots,w_{T} \in \R^D$ and
  $g_0,\ldots,g_{T-1} \in \R^D$ such that
  $w_0 = 0_D$,
  $w_{t+1} = w_t - g_t/(2C_0)$,
  $\|g_t - \nabla f(w_t)\|_2 \leq \epsilon$,
  $\|w_{t+1}\|_2 \leq R$ for all
  $t\in\{0,1,\ldots,T-1\}$ and
  \begin{align*}
    \bZ^{(t)} &\coloneqq \TF^{(t)} \circ \cdots \circ \TF^{(1)}
    \circ \TF^{(0)} \circ
    \Embed_\de \bigl((X_i, Y_i)_{i \in [n]}, X_{n+1}\bigr) \\
    &\,=
    \begin{pmatrix}
      \bX & Y &
      \bigl(\bX - 1_n X_{n+1}^\T\bigr)/h
      & n^{-1/2}\sqrt{K_h(\bX, X_{n+1})} & \check{\bX}
      & \check{Y} & 1_n w_t^\T
      & 1_n & 0_n \\
      X_{n+1}^\T & 0 & 0_d^\T & n^{-1/2}h^{-d/2} & a^\T & b & w_t^\T & 1 & 1
    \end{pmatrix},
  \end{align*}
  for $t\in \{0,\ldots,T\}$. We argue by induction.
  The case $T=0$ is verified by~\eqref{eq:GD-step-0}.
  Now suppose that the claim is true for some
  $T \leq \lfloor RC_0/\epsilon\rfloor -1$.
  We may apply Lemma~\ref{lemma:transformer-GD} to deduce that
  there exists $\TF^{(T+1)} \in \cT(\de,\dffn,1,B)$ such that
  \begin{align*}
    &\bZ^{(T+1)} \coloneqq
    \TF^{(T+1)}\bigl(\bZ^{(T)}\bigr)\\
    &=
    \begin{pmatrix}
      \bX & Y &
      \bigl(\bX - 1_n X_{n+1}^\T\bigr)/h
      & n^{-1/2}\sqrt{K_h(\bX, X_{n+1})} & \check{\bX}
      & \check{Y} & 1_n (w_T-g_T/(2C_0))^\T
      & 1_n & 0_n \\
      X_{n+1}^\T & 0 & 0_d^\T & n^{-1/2}h^{-d/2}
      & a^\T & b & (w_T-g_T/(2C_0))^\T & 1 & 1
    \end{pmatrix},
  \end{align*}
  where by~\eqref{eq:xi-def},
  the inductive hypothesis that $\|w_T\|_2 \leq R$ and
  Lemma~\ref{lemma:transformer-GD}, we have
  \begin{align*}
    \|g_T - \nabla f(w_T)\|_2 \leq 18D(M+3)(R+1)n\xi = \epsilon.
  \end{align*}
  Moreover, by Lemma~\ref{lemma:GD-error}
  (with $C_1\coloneqq 2/C_0$ and $C_2 \coloneqq 2C_0$ therein),
  and since $T+1 \leq \lfloor RC_0/\epsilon\rfloor$,
  we have $\|w_{T+1}\|_2 \leq R$. This proves the claim by induction. Now, let
  $T\coloneqq \lceil 4C_0^2\log n \rceil$ and further assume that $n\geq n_1$,
  where $n_1$ is large
  enough that $T \leq \lfloor RC_0/\epsilon \rfloor$
  (note that $RC_0/\epsilon$ is quartic
  in $n$ whereas $T$ is logarithmic).
  Then by Lemma~\ref{lemma:GD-error} again
  and~\eqref{eq:w_*-bound}, we deduce that
  \begin{align*}
    \|w_T - w_*\|_2
    \leq \exp\biggl(-\frac{T}{2C_0^2}\biggr) \cdot C_0(M+1)n\sqrt{D}
    + \frac{T\epsilon}{2C_0} \leq \frac{C_3}{n}.
  \end{align*}
  Let $\TF^{(T+1)} \in \cT(\de,\dffn,1,B)$ be a transformer block such
  that for any $\bZ\in R^{(n+1)\times \de}$,
  the $(d+1)$th column of $\TF^{(T+1)}(\bZ)$ is
  equal to the sum of the $(d+1)$th column and the
  $(2d+D+4)$th column of $\bZ$
  (this can be done by using only the FFN layer).
  Further, define
  $\TF \coloneqq \TF^{(T+1)} \circ \cdots \circ \TF^{(0)}
  \in \cT(\de,\dffn,L^{(0)}+T+1,B)$.
  Then, writing $w_{T,1}$ and~$w_{*,1}$
  for the first entries of $w_T$ and $w_*$ respectively,
  we have for $n \geq n_0 \vee n_1$ that
  \begin{align}
    \nonumber
    \bigl|
    \Read_{M,d} \circ \TF &\circ \Embed_\de
    \bigl((X_i, Y_i)_{i \in [n]}, X_{n+1}\bigr) - \tilde{m}_n(X_{n+1})
    \bigr|\\
    \label{eq:gradient_approx_final}
    &= \bigl|
    (-M) \vee w_{T,1} \wedge M - (-M) \vee w_{*,1} \wedge M
    \bigr|
    \leq |w_{T,1} - w_{*,1}| \leq \frac{C_3}{n},
  \end{align}
  with probability at least
  $1-n^{C_0/(2\alpha+d)} \exp\bigl(-n^{2\alpha/(2\alpha+d)}/C_0\bigr)$.
  On the other hand, the left-hand side of~\eqref{eq:gradient_approx_final}
  is bounded by $2M$, so by replacing $C_3$ with $C_4$ if necessary,
  the conclusion holds for all $n\in\N$.
  Finally, note that $L^{(0)}+T+1 \leq \lceil C_5 \log (en) \rceil$, so
  the lemma holds with $C\coloneqq C_0 \vee C_4 \vee C_5$.
\end{proof}

\section{Covering Numbers of Transformers}
\label{sec:covering}

In this section, we provide bounds
for the covering numbers of transformers with
linear attention.
Similar results on the covering numbers of
transformers with ReLU attention can be found in \citep{bai2023transformers}.

\begin{lemma}
  \label{lem:attention_lipschitz}
  Let $n, \de \in \N$ and $B, \delta, R > 0$.
  Suppose that
  $\bQ, \bK, \bV,
  \bQ', \bK', \bV' \in
  \R^{\de \times \de}$
  satisfy
  $\|\bQ\|_{\max}
  \lor
  \|\bK\|_{\max}
  \lor
  \|\bV\|_{\max}
  \lor
  \|\bQ'\|_{\max}
  \lor
  \|\bK'\|_{\max}
  \lor
  \|\bV'\|_{\max}
  \leq B$
  and
  $\|\bQ - \bQ'\|_{\max}
  \lor
  \|\bK - \bK'\|_{\max}
  \lor
  \|\bV - \bV'\|_{\max}
  \leq \delta$.
  Let $\bZ \in \R^{(n+1) \times \de}$ be such that
  $\|\bZ\|_{\max} \leq R$.
  Then the linear attention function from
  Definition~\ref{def:attention} satisfies
  \begin{align*}
    \bigl\|
    \Attn_{\bQ, \bK, \bV}(\bZ)
    - \Attn_{\bQ', \bK', \bV'}(\bZ)
    \bigr\|_{\max}
    &\leq
    3 B^2 R^3
    (n+1)^{3/2} \de^{9/2}
    \delta.
  \end{align*}
\end{lemma}

\begin{proof}[Proof of Lemma~\ref{lem:attention_lipschitz}]

  From Definition~\ref{def:attention},
  \begin{align*}
    \bigl\|
    \Attn_{\bQ, \bK, \bV}(\bZ)
    &- \Attn_{\bQ', \bK', \bV'}(\bZ)
    \bigr\|_{\max}
    =
    \bigl\|
    \bZ \bQ (\bZ \bK)^\T \bZ \bV
    - \bZ \bQ' (\bZ \bK')^\T \bZ \bV'
    \bigr\|_{\max} \\
    &\leq
    \bigl\|
    \bZ (\bQ - \bQ') (\bZ \bK)^\T \bZ \bV
    \bigr\|_{\op}
    + \bigl\|
    \bZ \bQ' \{\bZ (\bK - \bK')\}^\T \bZ \bV
    \bigr\|_{\op} \\
    &\hspace{6cm}+
    \bigl\|
    \bZ \bQ' (\bZ \bK')^\T \bZ (\bV - \bV')
    \bigr\|_{\op} \\
    &\leq
    \|\bZ\|_{\op}^3
    \|\bQ - \bQ'\|_{\op}
    \|\bK\|_{\op}
    \|\bV\|_{\op}
    + \|\bZ\|_{\op}^3
    \|\bQ'\|_{\op}
    \|\bK - \bK'\|_{\op}
    \|\bV\|_{\op} \\
    &\hspace{5cm}+
    \|\bZ\|_{\op}^3
    \|\bQ'\|_{\op}
    \|\bK'\|_{\op}
    \|\bV - \bV'\|_{\op} \\
    &\leq
    \{(n+1) \de\}^{3/2} \de^3
    \|\bZ\|_{\max}^3
    \|\bQ - \bQ'\|_{\max}
    \|\bK\|_{\max}
    \|\bV\|_{\max} \\
    &\hspace{2.5cm}+
    \{(n+1) \de\}^{3/2} \de^3
    \|\bZ\|_{\max}^3
    \|\bQ'\|_{\max}
    \|\bK - \bK'\|_{\max}
    \|\bV\|_{\max} \\
    &\hspace{2.5cm}+
    \{(n+1) \de\}^{3/2} \de^3
    \|\bZ\|_{\max}^3
    \|\bQ'\|_{\max}
    \|\bK'\|_{\max}
    \|\bV - \bV'\|_{\max} \\
    &\leq
    3 B^2 R^3
    (n+1)^{3/2} \de^{9/2}
    \delta,
  \end{align*}
  as required.
\end{proof}

\begin{lemma}
  \label{lem:ffn_lipschitz}
  Let $n, \de, \dffn \in \N$
  and $B, \delta, R > 0$.
  Suppose that
  $\bW_1, \bW_1' \in \R^{\dffn \times \de}$,
  $\bW_2, \bW_2' \in \R^{\de \times \dffn}$,
  $b_1, b_1' \in \R^{\dffn}$ and
  $b_2, b_2' \in \R^{\de}$
  satisfy
  $\|\bW_1\|_{\max}
  \lor
  \|\bW_2\|_{\max}
  \lor
  \|b_1\|_{\infty}
  \lor
  \|b_2\|_{\infty}
  \lor
  \|\bW_1'\|_{\max}
  \lor
  \|\bW_2'\|_{\max}
  \lor
  \|b_1'\|_{\infty}
  \lor
  \|b_2'\|_{\infty}
  \leq B$
  and
  $\|\bW_1 - \bW_1'\|_{\max}
  \lor
  \|\bW_2 - \bW_2'\|_{\max}
  \lor
  \|b_1 - b_1'\|_{\infty}
  \lor
  \|b_2 - b_2'\|_{\infty}
  \leq \delta$.
  Let $\bZ \in \R^{(n+1) \times \de}$ be such that
  $\|\bZ\|_{\max} \leq R$.
  Then the feed-forward network function from
  Definition~\ref{def:ffn} satisfies
  \begin{align*}
    \bigl\|
    \FFN_{\bW_1, \bW_2, b_1, b_2}(\bZ)
    - \FFN_{\bW_1', \bW_2', b_1', b_2'}(\bZ)
    \bigr\|_{\max}
    \leq
    2 (B + 1) (R + 1) (n + 1) \de^{3/2} \dffn^{3/2} \delta.
  \end{align*}

\end{lemma}

\begin{proof}[Proof of Lemma~\ref{lem:ffn_lipschitz}]
  From Definition~\ref{def:ffn},
  since $\ReLU$ is $1$-Lipschitz from
  $\|\cdot\|_{\max}$ to $\|\cdot\|_{\max}$,
  \begin{align*}
    &\bigl\|
    \FFN_{\bW_1, \bW_2, b_1, b_2}(\bZ)
    - \FFN_{\bW_1', \bW_2', b_1', b_2'}(\bZ)
    \bigr\|_{\max} \\
    &\quad=
    \bigl\|
    \ReLU\bigl(
      \bZ \bW_1^\T + 1_{n+1} b_1^\T
    \bigr) \bW_2^\T
    + 1_{n+1} b_2^{\T} -
    \ReLU\bigl(
      \bZ \bW_1^{\prime \T} + 1_{n+1} b_1^{\prime \T}
    \bigr) \bW_2^{\prime \T}
    - 1_{n+1} b_2'^{\T}
    \bigr\|_{\max} \\
    &\quad\leq
    \bigl\|
    \ReLU\bigl(
      \bZ \bW_1^\T + 1_{n+1} b_1^\T
    \bigr) \bW_2^\T
    - \ReLU\bigl(
      \bZ \bW_1^{\prime \T} + 1_{n+1} b_1^{\prime \T}
    \bigr) \bW_2^\T
    \bigr\|_{\max} \\
    &\qquad+
    \bigl\|
    \ReLU\bigl(
      \bZ \bW_1^{\prime \T} + 1_{n+1} b_1^{\prime \T}
    \bigr) \bW_2^\T
    - \ReLU\bigl(
      \bZ \bW_1^{\prime \T} + 1_{n+1} b_1^{\prime \T}
    \bigr) \bW_2^{\prime \T}
    \bigr\|_{\max} \!
    + \|b_2 - b_2'\|_{\infty} \\
    &\quad\leq
    \|\bW_2^\T\|_{\op}
    \bigl\|
    \ReLU\bigl(
      \bZ \bW_1^\T + 1_{n+1} b_1^\T
    \bigr)
    - \ReLU\bigl(
      \bZ \bW_1^{\prime \T} + 1_{n+1} b_1^{\prime \T}
    \bigr)
    \bigr\|_{\op} \\
    &\qquad+
    \bigl\|
    \ReLU\bigl(
      \bZ \bW_1^{\prime \T} + 1_{n+1} b_1^{\prime \T}
    \bigr)
    \bigr\|_{\op}
    \| \bW_2^\T - \bW_2^{\prime \T} \|_{\op}
    + \|b_2 - b_2'\|_{\infty} \\
    &\quad\leq
    \sqrt{\dffn \de}
    B
    \sqrt{(n + 1) \dffn}
    \bigl\|
    \bZ \bW_1^\T + 1_{n+1} b_1^\T
    - \bZ \bW_1^{\prime \T} - 1_{n+1} b_1^{\prime \T}
    \bigr\|_{\max} \\
    &\qquad+
    \sqrt{(n + 1) \dffn}
    \bigl\|
    \bZ \bW_1^{\prime \T} + 1_{n+1} b_1^{\prime \T}
    \bigr\|_{\max}
    \sqrt{\dffn \de}
    \,\delta
    + \delta \\
    &\quad\leq
    B
    \dffn
    \sqrt{(n + 1) \de}
    \bigl(
      \bigl\|
      \bZ \bW_1^\T - \bZ \bW_1^{\prime \T}
      \bigr\|_{\max}
      + \| b_1 - b_1' \|_{\infty}
    \bigr)
    \iftoggle{journal}{}{\\ &\qquad}
    +
    \delta
    \dffn
    \sqrt{(n + 1) \de}
    \bigl(
      \bigl\|
      \bZ \bW_1^{\prime \T}
      \bigr\|_{\max}
      + \|b_1'\|_{\infty}
    \bigr)
    + \delta \\
    &\quad\leq
    B
    \dffn
    \sqrt{(n + 1) \de}
    \bigl(
      R
      \delta
      \de
      \sqrt{(n + 1) \dffn}
      + \delta
    \bigr) +
    \delta
    \dffn
    \sqrt{(n + 1) \de}
    \bigl(
      R
      B
      \de
      \sqrt{(n + 1) \dffn}
      + B
    \bigr)
    + \delta \\
    &\quad=
    2 B R \delta (n + 1) \de^{3/2} \dffn^{3/2}
    + 2 B \delta \sqrt{(n + 1) \de} \dffn
    + \delta \\
    &\quad\leq
    2 (B + 1) (R + 1) (n + 1) \de^{3/2} \dffn^{3/2} \delta,
  \end{align*}
  as required.
\end{proof}

\begin{lemma}
  \label{lem:transformer_lipschitz}
  Let $n, d, \de, \dffn, L \in \N$
  with $\de \geq d + 2$
  and take $B, \delta, R > 0$.
  For $\ell \in [L]$, let
  $\btheta^{(\ell)}, \btheta^{\prime (\ell)}
  \in
  \R^{\de \times \de}
  \times
  \R^{\de \times \de}
  \times
  \R^{\de \times \de}
  \times
  \R^{\de \times \dffn}
  \times
  \R^{\dffn \times \de }
  \times
  \R^{\dffn}
  \times
  \R^{\de}$.
  Let $\btheta \coloneqq (\btheta^{(\ell)})_{\ell \in [L]}$
  and $\btheta' \coloneqq (\btheta^{\prime (\ell)})_{\ell \in [L]}$.
  Suppose that every entry of each parameter in
  $\btheta$ and $\btheta'$ is bounded by $B$,
  and that every entry of each parameter in
  $\btheta - \btheta'$ is bounded by $\delta$.
  Let $\bZ \in \R^{(n+1) \times \de}$ be such that
  $\|\bZ\|_{\max} \leq R$.
  Then the transformer function from
  Definition~\ref{def:transformer} satisfies
  \begin{align*}
    \bigl\|
    \TF_{\btheta} (\bZ)
    - \TF_{\btheta'} (\bZ)
    \bigr\|_{\max}
    &\leq
    6^L
    (B + 1)^{3 L}
    (R + 1)^{4 L}
    (n + 1)^{5L/2}
    \de^{6L}
    \dffn^{3L/2}
    \delta.
  \end{align*}
  Further, if $(x_i, y_i) \in [0, 1]^d \times [-M-1, M+1]$
  for each $i \in [n]$ and $x_{n+1} \in [0, 1]^d$, then
  \begin{align*}
    \bigl|
    \Read_{M,d} \circ \TF_{\btheta} \circ \Embed_{\de}
    \bigl((x_i, y_i)_{i \in [n]}, x_{n+1}\bigr) &-
    \Read_{M,d} \circ \TF_{\btheta'} \circ \Embed_{\de}
    \bigl((x_i, y_i)_{i \in [n]}, x_{n+1}\bigr)
    \bigr| \\
    &\leq
    6^L
    (B + 1)^{3 L}
    (M + 2)^{4 L}
    (n + 1)^{5L/2}
    \de^{6L}
    \dffn^{3L/2}
    \delta.
  \end{align*}
\end{lemma}

\begin{proof}[Proof of Lemma~\ref{lem:transformer_lipschitz}]
  By Lemmas~\ref{lem:attention_lipschitz}
  and~\ref{lem:ffn_lipschitz},
  for $\ell \in [L]$,
  by Lipschitz composition,
  \begin{align*}
    \bigl\|
    \Block_{\btheta^{(\ell)}} (\bZ)
    - \Block_{\btheta^{\prime (\ell)}} (\bZ)
    \bigr\|_{\max}
    &\leq
    6
    (B + 1)^3
    (R + 1)^4
    (n + 1)^{5/2}
    \de^{6}
    \dffn^{3/2}
    \delta.
  \end{align*}
  Therefore, by composition over the $L$ layers,
  \begin{align*}
    \bigl\|
    \TF_{\btheta} (\bZ)
    - \TF_{\btheta'} (\bZ)
    \bigr\|_{\max}
    &\leq
    6^L
    (B + 1)^{3 L}
    (R + 1)^{4 L}
    (n + 1)^{5L/2}
    \de^{6L}
    \dffn^{3L/2}
    \delta.
  \end{align*}
  The second result follows from the first, on observing that
  the output of
  $\Embed_{\de}$ is bounded by $M + 1$, and
  $\Read_{M,d}$ is $1$-Lipschitz from
  $\|\cdot\|_{\max}$ to $|\cdot|$.
\end{proof}

\begin{lemma}
  \label{lem:covering}
  Let $n, d, \de, \dffn, L, \Gamma \in \N$
  with $\de \geq d + 2$ and take $B, M > 0$
  and $\delta \in [0, M]$.
  Define the $\delta$-covering number in
  $\|\cdot\|_\infty$-norm of the class $\cF(\de,\dffn,L,B,M)$
  of functions from
  $\bigl([0,1]^d \times \R\bigr)^n \times [0, 1]^d$ to $\R$ by
  \begin{align*}
    &N\bigl(
      \cF(\de,\dffn,L,B,M),
      \delta,
      \|\cdot\|_\infty
    \bigr) \\
    &\qquad\coloneqq
    \inf \biggl\{
      |\cF'| :
      \cF' \subseteq \cF(\de,\dffn,L,B,M),
      \ \sup_{f \in \cF(\de,\dffn,L,B,M)}
      \inf_{f' \in \cF'}
      \|f - f'\|_\infty
      \leq \delta
    \biggr\}.
  \end{align*}
  Then the log-covering number (entropy) satisfies
  \begin{align*}
    \log N\bigl(
      \cF(\de,\dffn,L,B,M)&,
      \delta,
      \|\cdot\|_\infty
    \bigr) \leq
    24 L \de (\de + \dffn)
    \log \bigl(
      (B + 1)^{L} (M + 2)^{2L}
      (n + 1)^{L} \de^{L} \dffn^{L}
      / \delta
    \bigr).
  \end{align*}
\end{lemma}

\begin{proof}[Proof of Lemma~\ref{lem:covering}]
  Define $\Theta \! \coloneqq \!
  \bigl(
    [-B, B]^{\de \times \de}
    \times
    [-B, B]^{\de \times \de}
    \times
    [-B, B]^{\de \times \de}
    \times
    [-B, B]^{\de \times \dffn}
    \times
    [-B, B]^{\dffn \times \de }
    \times
    [-B, B]^{\dffn}
    \times
    [-B, B]^{\de}
  \bigr)^L$
  and take
  $C \coloneqq 6^L (B + 1)^{3 L} (M + 2)^{4 L}
  (n + 1)^{5L/2} \de^{6L} \dffn^{3L/2}$.
  The dimension of $\Theta$ is
  $L(3 \de^2 + 2 \de\dffn + \de + \dffn)
  \leq 4 L \de (\de + \dffn)$.
  Let $\Theta'$ be a $\delta / C$-cover
  of $(\Theta, \|\cdot\|_\infty)$
  of cardinality at most
  $\{2 (B+1) C / \delta\}^{4 L \de (\de + \dffn)}$.
  Take $\cF' \coloneqq
  \bigl\{\Read_{M,d} \circ \TF_{\btheta} \circ \Embed_{\de}:
  \btheta \in \Theta'\bigr\}$ so
  by Lemma~\ref{lem:transformer_lipschitz},
  $\cF'$ is a $\delta$-cover of
  $\cF(\de,\dffn,L,B,M)$ in $\|\cdot\|_\infty$-norm
  with
  \begin{align*}
    \log |\cF'|
    &\leq
    \log
    \bigl(
      \{2 (B+1) C / \delta\}^{4 L \de (\de + \dffn)}
    \bigr)
    =
    4 L \de (\de + \dffn)
    \log \{2 (B+1) C / \delta\} \\
    &=
    4 L \de (\de + \dffn)
    \log \bigl(
      2 \cdot
      6^L (B + 1)^{3 L + 1} (M + 2)^{4 L}
      (n + 1)^{5L/2} \de^{6L} \dffn^{3L/2}
      / \delta
    \bigr) \\
    &\leq
    4 L \de (\de + \dffn)
    \log \bigl(
      (B + 1)^{4 L} (M + 2)^{5 L}
      (n + 1)^{11L/2} \de^{6L} \dffn^{3L/2}
      / \delta
    \bigr) \\
    &\leq
    24 L \de (\de + \dffn)
    \log \bigl(
      (B + 1)^{L} (M + 2)^{2L}
      (n + 1)^{L} \de^{L} \dffn^{L}
      / \delta
    \bigr),
  \end{align*}
  as required.
\end{proof}

\section{Risk Decomposition}
\begin{theorem}
  \label{thm:statistical}
  Let $n, d, \de, \dffn, L, \Gamma \in \N$
  with $\de \geq d + 2$ and take $B > 0$.
  Take $\hat f_\Gamma$ as in Definition~\ref{def:estimator}.
  There exists a universal constant $C > 0$ such that
  \begin{align*}
    \E\bigl\{R(\hat f_\Gamma)\bigr\}
    &\leq
    \inf_{f \in \cF(\de, \dffn, L, B,M)}
    2 R(f) - \sigma^2 +
    C (M + 1)^5
    L \de (\de + \dffn)
    \frac{
      L \log \bigl\{
        (B + 1)
        n \de \dffn
      \bigr\}
      + \log \Gamma
    }{\Gamma}.
  \end{align*}
\end{theorem}

\begin{proof}[Proof of Theorem~\ref{thm:statistical}]
  Throughout the proof, we write
  $C_1, C_2, \ldots$ to denote
  universal positive constants.
  Let $f^\star \in \cF(\de, \dffn, L, B,M)$
  satisfy
  $R(f^\star) \leq
  \inf_{f \in \cF(\de, \dffn, L, B,M)} R(f) + 1/\Gamma$.
  By the Doob--Dynkin lemma
  \citep[e.g.,][Lemma 1.14]{kallenberg2021foundations}, let
  $f_0: \bigl([0,1]^d \times \R\bigr)^n \times [0,1]^d
  \rightarrow \R$ be a Borel-measurable function satisfying
  $f_0\bigl(\cD_n, X_{n+1}\bigr) = \E(Y_{n+1} \mid \cD_n, X_{n+1})$
  almost surely, so that $R(f_0) \geq \sigma^2$.
  Since $\hat R_\Gamma(\hat f_\Gamma) \leq \hat R_\Gamma(f^\star)$
  and as $\E\bigl\{\hat R_\Gamma(f)\bigr\} = R(f)$ for all
  bounded, measurable functions $f$, we have
  \begin{align}
    \nonumber
    \E\bigl\{
      R(\hat f_\Gamma)
    \bigr\}
    &=
    R(f_0)
    + \E\bigl\{
      R(\hat f_\Gamma)
      - R(f_0)
      - 2 \hat R_\Gamma(\hat f_\Gamma)
      + 2 \hat R_\Gamma(f_0)
    \bigr\}
    + 2 \, \E\bigl\{
      \hat R_\Gamma(\hat f_\Gamma)
      - \hat R_\Gamma(f_0)
    \bigr\} \\
    \nonumber
    &\quad\leq
    R(f_0)
    + \E\bigl\{
      R(\hat f_\Gamma)
      - R(f_0)
      - 2 \hat R_\Gamma(\hat f_\Gamma)
      + 2 \hat R_\Gamma(f_0)
    \bigr\}
    + 2 \bigl\{
      R(f^\star)
      - R(f_0)
    \bigr\} \\
    \label{eq:statistical_decomposition}
    &\quad\leq
    \E\bigl\{
      R(\hat f_\Gamma)
      - R(f_0)
      - 2 \hat R_\Gamma(\hat f_\Gamma)
      + 2 \hat R_\Gamma(f_0)
    \bigr\}
    + \inf_{f \in \cF(\de, \dffn, L, B,M)}
    2 R(f) + \frac{2}{\Gamma} - \sigma^2.
  \end{align}
  To bound the expectation, we apply the bound given by
  \citet[Theorem~11.4]{gyorfi2002distribution} with $\epsilon=1/2$
  and $\alpha=\beta=t/2$ therein,
  noting that $|Y_{n+1}^{(\gamma)}| \leq M + 1$
  for each $\gamma \in [\Gamma]$ and
  $\sup_{f \in \cF(\de, \dffn, L, B,M)} \|f\|_\infty \leq M$.
  As $L_1$-covering numbers are bounded by
  $L_\infty$-covering numbers,
  there exists a universal constant $C_1 > 0$
  such that for all $t > 0$,
  \begin{align*}
    &\P\Bigl(
      R(\hat f_\Gamma) - R(f_0)
      - 2 \hat R_\Gamma(\hat f_\Gamma) + 2 \hat R_\Gamma(f_0)
      \geq
      t
    \Bigr) \\
    &\qquad\leq
    C_1 N\biggl(
      \cF(\de,\dffn,L,B,M),
      \frac{t}{C_1(M+1)},
      \|\cdot\|_\infty
    \biggr)
    \exp\biggl(- \frac{\Gamma t}{C_1(M+1)^4}\biggr).
  \end{align*}
  Hence, by Lemma~\ref{lem:covering},
  taking $t \coloneqq C_1 s (M+1)^4/\Gamma$
  for $s \in [1,\Gamma]$,
  \begin{align}
    \nonumber
    &\P\biggl(
      R(\hat f_\Gamma) - R(f_0)
      - 2 \hat R_\Gamma(\hat f_\Gamma) + 2 \hat R_\Gamma(f_0)
      \geq
      \frac{C_1 s (M + 1)^4}{\Gamma}
    \biggr) \\
    \nonumber
    &\qquad\leq
    C_1 N\bigl(
      \cF(\de,\dffn,L,B,M),
      s / \Gamma,
      \|\cdot\|_\infty
    \bigr)
    e^{-s} \\
    \label{eq:gyorfi_bound}
    &\qquad\leq
    C_1
    \exp\Bigl(
      24 L \de (\de + \dffn)
      \log \bigl\{
        (B + 1)^{L} (M + 2)^{2L}
        (n + 1)^{L} \de^{L} \dffn^{L}
        \Gamma
      \bigr\}
    \Bigr)
    e^{-s}.
  \end{align}
  In fact, taking $C_1 \geq 12$ without loss of generality,
  \eqref{eq:gyorfi_bound} holds for all
  $s \in [1, \infty)$ because $\|\hat f_\Gamma\|_\infty \leq M$
  and $\|f_0\|_\infty \leq M + 1$, so
  $R(\hat f_\Gamma) \leq 4(M+1)^2$
  and $\hat R_\Gamma(f_0) \leq 4(M+1)^2$.
  Therefore, for all $u > 0$,
  with probability at least
  $1 - C_2 e^{-u}$,
  as $\de \geq 3 > e$,
  \begin{align*}
    & R(\hat f_\Gamma) - R(f_0)
    - 2 \hat R_\Gamma(\hat f_\Gamma) + 2 \hat R_\Gamma(f_0) \\
    &\qquad\leq
    \frac{C_2 (M + 1)^4}{\Gamma}
    \Bigl\{
      u +
      L \de (\de + \dffn)
      \log \bigl\{
        (B + 1)^{L} (M + 2)^{2L}
        (n + 1)^{L} \de^{L} \dffn^{L}
        \Gamma
      \bigr\}
    \Bigr\}.
  \end{align*}
  Integrating the tail probability, we obtain
  \begin{align}
    \nonumber
    &\E\bigl\{
      R(\hat f_\Gamma)
      - R(f_0)
      - 2 \hat R_\Gamma(\hat f_\Gamma)
      + 2 \hat R_\Gamma(f_0)
    \bigr\} \\
    \nonumber
    &\qquad\leq
    \frac{C_3 (M + 1)^4
      L \de (\de + \dffn)
      \log \bigl\{
        (B + 1)^{L} (M + 2)^{2L}
        (n + 1)^{L} \de^{L} \dffn^{L}
        \Gamma
      \bigr\}
    }{\Gamma} \\
    \label{eq:statistical_expectation}
    &\qquad\leq
    C_4 (M + 1)^5
    L \de (\de + \dffn)
    \frac{
      L \log \bigl\{
        (B + 1)
        n \de \dffn
      \bigr\}
      + \log \Gamma
    }{\Gamma}.
  \end{align}
  The result follows from~\eqref{eq:statistical_decomposition}
  and \eqref{eq:statistical_expectation}.
\end{proof}

\section{Proofs of Main Results}

\begin{proof}[Proof of Theorem~\ref{thm:approximation}]
  Throughout the proof, $C'_1, C'_2 > 0$ are quantities
  depending only on
  $d$, $\alpha$, $M$, $c_X$ and $C_X$.
  By Proposition~\ref{prop:TF-approx-lp},
  since $n^{(p + 1) / (2\alpha + d)} \leq n^2$,
  there exists a transformer
  $\TF \in \cT(\de,\dffn,L,B)$
  such that if
  $f_\TF \coloneqq
  \Read_{M,d} \circ \TF \circ \Embed_{\de}
  \in \cF(\de,\dffn,L,B,M)$,
  then
  \begin{align*}
    \bigl|R(
      \iftoggle{journal}{}{&}
    f_\TF)
    \iftoggle{journal}{&}{}
    - R(f_\LocPol)\bigr|
    =
    \Bigl|
    \E \Bigl(
      \bigl\{
        Y_{n+1} - f_\TF(\cD_n,X_{n+1}\bigr)
      \bigr\}^2
    \Bigr) - \E \Bigl(
      \bigl\{
        Y_{n+1} - f_\LocPol(\cD_n,X_{n+1}\bigr)
      \bigr\}^2
    \Bigr)
    \Bigr| \\
    &=
    \Bigl|
    \E \Bigl(
      \bigl\{
        f_\LocPol(\cD_n,X_{n+1}\bigr) - f_\TF(\cD_n,X_{n+1}\bigr)
      \bigr\} \bigl\{2Y_{n+1} - f_\TF(\cD_n,X_{n+1}\bigr) -
      f_\LocPol(\cD_n,X_{n+1}\bigr) \bigr\}
    \Bigr)
    \Bigr| \\
    &\leq (4M+2) \,
    \E \Bigl(
      \bigl|
      f_\TF(\cD_n,X_{n+1}\bigr) - f_\LocPol(\cD_n,X_{n+1}\bigr)
      \bigr|
    \Bigr) \\
    &\leq (4M+2)\Bigl\{ C_1'/n + 2M n^{C_1'/(2\alpha+d)}
    \exp\bigl(-n^{2\alpha/(2\alpha+d)}/C_1'\bigr) \Bigr\}
    \leq \frac{C_2'}{n}.
  \end{align*}
\end{proof}

\begin{proof}[Proof of Theorem~\ref{thm:main_result}]
  Throughout the proof, $C'_1, C'_2, \ldots > 0$ are quantities
  depending only on
  $d$, $\alpha$, $M$, $c_X$ and $C_X$.
  Let $f_{\LocPol}(\cD_n, X_{n+1})
  \coloneqq \tilde m_n(X_{n+1})$
  be the truncated local polynomial estimator as defined in
  Appendix~\ref{sec:locpol}
  with degree $p$,
  bandwidth $h = n^{-1/(2\alpha+d)}$
  and kernel $K(x) \coloneqq (1 - \|x\|_1)_+^2$.
  By Theorem~\ref{thm:locpol}, we have
  \begin{align}
    R(f_\LocPol)
    &=
    \E \bigl(
      \bigl\{
        Y_{n+1} - f_\LocPol(\cD_n,X_{n+1})
      \bigr\}^2
    \bigr) \nonumber\\
    &=
    \E (\varepsilon_{n+1}^2)
    +
    \E \biggl(
      \int_{[0, 1]^d}
      \bigl\{
        \tilde m_n(x)
        -
        m(x)
      \bigr\}^2
      f_X(x)
      \,\diff x
    \biggr)
    \leq \sigma^2
    + C'_1 n^{\frac{-2\alpha}{2\alpha + d}}. \label{eq:locpol-bound}
  \end{align}
  %
  %
  By Theorem~\ref{thm:statistical},
  since $R(f_\TF) \geq \inf_{f \in \cF(\de,\dffn,L,B,M)} R(f)$,
  we have
  \begin{align*}
    \E\bigl\{R(\hat f_\Gamma)\bigr\}
    &\leq
    2 R(f_\TF) - \sigma^2
    + C'_4 (M + 1)^5
    L \de (\de + \dffn)
    \frac{
      L \log \bigl\{
        (B + 1)
        n \de \dffn
      \bigr\}
      + \log \Gamma
    }{\Gamma}.
  \end{align*}
  Thus, by Theorem~\ref{thm:approximation}, with
  $\de \coloneqq 2d+2D+5$,
  $\dffn \coloneqq 3(D+1)(28+2p)$,
  $L=\lceil C\log (en) \rceil$,
  $B=C n^2$ and
  $\Gamma \geq
  C n ^{2\alpha/(2\alpha+d)}\log^3(e n)$, we have
  \begin{align*}
    \E\bigl\{R(\hat f_\Gamma)\bigr\}
    - \sigma^2
    &\leq
    2 \bigl(R(f_\TF) - \sigma^2\bigr)
    + C'_5
    \frac{ \log^3(e n) + \log (e n) \log \Gamma}{\Gamma} \\
    &\leq
    2 \bigl(R(f_\LocPol) - \sigma^2\bigr)
    + \frac{C'_3}{n}
    + C'_5
    \frac{ \log^3(e n) + \log(e n) \log \Gamma}{\Gamma}
    \leq
    C'_6 n^{\frac{-2\alpha}{2\alpha + d}},
  \end{align*}
  where
  the final inequality follows from~\eqref{eq:locpol-bound}.
\end{proof}



\section{Approximating Linear Attention by Softmax}
\label{sec:softmax-approx-linear}
For $j\in[n]$, $x\in\R^{n}$ and $t \geq 0$, define
$g_j(t) \coloneqq \mathrm{softmax}(t x)_j
= e^{t x_j} / \bigl(\sum_{i=1}^{n} e^{t x_i}\bigr)$
to be the $j$th coordinate of the softmax transformation
of $t x$.  Then
\begin{align*}
  g_j'(t) = \frac{x_j e^{t x_j} \bigl(\sum_{i=1}^{n} e^{t x_i}\bigr)
    - e^{t x_j} \bigl(\sum_{i=1}^{n} x_i e^{t
  x_i}\bigr)}{\bigl(\sum_{i=1}^{n} e^{t x_i}\bigr)^2},
\end{align*}
so $g_j'(0) = x_j/n - \sum_{i=1}^{n} x_i/ n^2$. By Taylor's theorem,
\begin{align*}
  \mathrm{softmax}(t x)_j = g_j(t) = g_j(0) + t g_j'(0) + O(t^2) = \frac{1}{n} +
  t\biggl\{\frac{x_j}{n} - \frac{\sum_{i=1}^{n} x_i}{n^2}\biggr\} + O(t^2),
\end{align*}
as $t\to 0$.
Hence, for $\bm Z \in \R^{n\times\de}$ and $\bm{Q},\bm K,\bm V \in
\R^{\de\times\de}$, we have
\begin{align}
  \mathrm{softmax}\bigl(t\bm Z \bm Q(\bm Z\bm K)^\T\bigr)\bm Z \bm V
  = \frac{\bm{1}_{n\times n} \bm Z \bm V}{n} + t\biggl\{\frac{\bm Z
    \bm Q(\bm Z\bm K)^\T \bm Z \bm V}{n} - \frac{\bm Z \bm Q(\bm Z\bm
  K)^\T \bm{1}_{n\times n} \bm Z \bm V}{n^2} \biggr\} + O(t^2).
  \label{eq:softmax-taylor-expansion}
\end{align}
Below, we outline how to construct a transformer with two softmax
attention layers followed by $O(\log n)$ FFN layers such that if the
input matrix is $\bm Z' \coloneqq (\bm Z,\, \bm 0_{n\times2\de}) \in
\R^{n\times 3\de}$, then its output approximates the output
$\bigl(\bm Z + \bm Z \bm Q(\bm Z\bm K)^\T \bm Z \bm V,\, 0_{n\times2\de}\bigr)$
of a linear attention layer.


\textbf{First attention layer:} We take $\bm Q^{(1)} = \bm K^{(1)}
\coloneqq \bm 0_{3\de\times3\de}$ and
\begin{align*}
  \bm V^{(1)} \coloneqq
  \begin{pmatrix}
    \bm 0_{\de\times\de} & n \bm I_{\de} & \bm 0_{\de\times\de}\\
    \bm 0_{2\de\times\de} & \bm 0_{2\de\times\de} & \bm 0_{2\de\times\de}
  \end{pmatrix}.
\end{align*}
Then
\begin{align*}
  \Attn^{(1)}(\bm Z')
  \coloneqq \bm Z' + \mathrm{softmax}\bigl(\bm Z' \bm Q^{(1)}(\bm
  Z'\bm K^{(1)})^\T\bigr)\bm Z' \bm V^{(1)} =
  \begin{pmatrix}
    \bm Z & \bm 1_{n\times n} \bm Z & \bm 0_{n\times\de}
  \end{pmatrix}.
\end{align*}

\textbf{Second attention layer:} We take
\begin{gather*}
  \bm Q^{(2)} \coloneqq
  \begin{pmatrix}
    \bm 0_{\de\times2\de} & t\bm Q \\
    \bm 0_{2\de\times2\de} & \bm 0_{2\de\times\de}
  \end{pmatrix},
  \qquad
  \bm K^{(2)} \coloneqq
  \begin{pmatrix}
    \bm 0_{\de\times2\de} & \bm K \\
    \bm 0_{2\de\times2\de} & \bm 0_{2\de\times\de}
  \end{pmatrix} \quad\text{and}\quad
  \bm V^{(2)}
  \begin{pmatrix}
    \bm 0_{\de\times2\de} & \frac{n}{t}\bm V \\
    \bm 0_{2\de\times2\de} & \bm 0_{2\de\times\de}
  \end{pmatrix}.
\end{gather*}
Then, writing $\bm Z^{(1)} \coloneqq \Attn^{(1)}(\bm Z') = (\bm Z,\,
\bm 1_{n\times n} \bm Z,\, \bm 0_{n\times\de})$ for the output of the
first attention layer, we have by~\eqref{eq:softmax-taylor-expansion} that
\begin{align*}
  \Attn^{(2)}(\bm Z^{(1)}) & \coloneqq \bm Z^{(1)} +
  \mathrm{softmax}\bigl(\bm Z^{(1)} \bm Q^{(2)} (\bm Z^{(1)}\bm
  K^{(2)})^\T\bigr)\bm Z^{(1)} \bm V^{(2)}\\
  &\phantom{:}=
  \begin{pmatrix}
    \bm Z & \bm 1_{n\times n} \bm Z & \frac{\bm{1}_{n\times n} \bm Z
    \bm V}{t} + \bm Z \bm Q(\bm Z\bm K)^\T \bm Z \bm V - \frac{\bm Z
    \bm Q(\bm Z\bm K)^\T \bm{1}_{n\times n} \bm Z \bm V}{n} + O(t)
  \end{pmatrix}.
\end{align*}
Define $S \coloneqq \bm Z^\T 1_n \in \R^{\de}$. Then 
\begin{align*}
  \Attn^{(2)}(\bm Z^{(1)})_{i,:} =
  \begin{pmatrix}
    Z_i^\T & S^\T & \frac{S^\T \bm V}{t} + \bigl(\bm Z \bm Q(\bm Z\bm
    K)^\T \bm Z \bm V\bigr)_i - \frac{Z_i^\T \bm Q \bm K^\T SS^\T \bV}{n} + O(t)
  \end{pmatrix}.
\end{align*}

\textbf{FFN layers:} We take the FFN layers such that for $z,S,u \in \R^{\de}$,
\begin{align*}
  \mathrm{FFN}\bigl((z^\T,\,S^\T,\,u^\T)\bigr)
  \approx
  \begin{pmatrix}
    z^\T + u^\T - \frac{S^\T \bm V}{t} + \frac{z^\T \bm Q \bm K^\T
    SS^\T \bV}{n} & 0_{\de}^\T & 0_{\de}^\T
  \end{pmatrix}.
\end{align*}
Note that the right-hand side of the above equation only involves
multiplication and summation. Thus, if we want the approximation
error to be $O(n^{-C})$ for some $C>0$, then it suffices to take the
width of the FFN layers to be constant in $n$ and the depth to be
$\Theta(\log n)$; see Lemma~\ref{lemma:NN-multiplication-general}.
Finally, taking $t=n^{-C}$ for some $C>0$ large enough, we deduce that
\begin{align*}
  \mathrm{FFN} \circ \mathrm{Attn}^{(2)} \circ \mathrm{Attn}^{(1)}(\bm Z') =
  \begin{pmatrix}
    \bm Z + \bm Z \bm Q(\bm Z\bm K)^\T \bm Z \bm V & \bm 0_{n\times2\de}
  \end{pmatrix} + O(n^{-1}).
\end{align*}

\end{document}